%% file: S2DF.tex
\documentclass[journal]{IEEEtran}
\pdfoutput=1
\usepackage{amsmath,amsfonts}
\usepackage{algorithmic}
\usepackage{algorithm}
\usepackage{array}
\usepackage[caption=false,font=normalsize,labelfont=sf,textfont=sf]{subfig}
\usepackage{textcomp}
\usepackage{stfloats}
\usepackage{url}
\usepackage{verbatim}
\usepackage{graphicx}
\usepackage{multirow}
\usepackage{overpic}
\usepackage{wrapfig}
\usepackage{booktabs}
\usepackage[switch]{lineno}

\usepackage{amsthm}
\usepackage[citecolor=blue, colorlinks]{hyperref}
\newtheorem{theorem}{Theorem}

\usepackage{natbib}
\usepackage{pdfpages}
\setcitestyle{numbers,square}

\usepackage{soul}

\newcommand{\revise}[1]{\textcolor{black}{#1}}

\begin{document}

\title{Monge-Ampere Regularization for Learning Arbitrary Shapes from Point Clouds}

\author{Chuanxiang Yang, Yuanfeng Zhou, Guangshun Wei, Long Ma, Junhui Hou,~\IEEEmembership{Senior Member,~IEEE},\\ Yuan Liu and Wenping Wang,~\IEEEmembership{Fellow,~IEEE}

\thanks{ This work was supported in part by National Natural Science Foundation of China under Grant 62172257, in part by Natural Science Foundation of Shandong Province under Grant ZR2024ZD12, in part by the NSFC Excellent Young Scientists Fund 62422118, and in part by the Hong Kong RGC under Grants 11219324 and 11219422. \textit{(Corresponding author: Yuanfeng Zhou.) }}

\thanks{Chuanxiang Yang, Yuanfeng Zhou, Guangshun Wei, Long Ma are with the School of Software, Shandong University, Jinan 250100, China (e-mail: chxyang2023@gmail.com; yfzhou@sdu.edu.cn; guangshunwei@gmail.com; malong@sdu.edu.cn).}
\thanks{Junhui Hou is with the Department of Computer Science,  City University of Hong Kong, Hong Kong (e-mail:jh.hou@cityu.edu.hk).}
\thanks{Yuan Liu is with the College of Computing and Data Science, Nanyang Technological University, Singapore (e-mail: liuyuanwhuer@gmail.com).}
\thanks{Wenping Wang is with the Department of Computer Science and Engineering, Texas A\&M University, College Station, TX 77843 USA (e-mail: wenping@tamu.edu).}}

\markboth{Journal of \LaTeX\ Class Files,~Vol.~14, No.~8, August~2021}
{Shell \MakeLowercase{\textit{et al.}}: A Sample Article Using IEEEtran.cls for IEEE Journals}

\IEEEpubid{0000--0000/00\$00.00~\copyright~2024 IEEE}
\maketitle

\begin{abstract}
As commonly used implicit geometry representations, the signed distance function (SDF) is limited to modeling watertight shapes, while the unsigned distance function (UDF) is capable of representing various surfaces. However, its inherent theoretical shortcoming, i.e., the non-differentiability at the zero-level set, would result in sub-optimal reconstruction quality.
In this paper, we propose the scaled-squared distance function (S\textsuperscript{2}DF), a novel implicit surface representation for modeling \textit{arbitrary} surface types. S\textsuperscript{2}DF does not distinguish between inside and outside regions while effectively addressing the non-differentiability issue of UDF at the zero-level set. We demonstrate that S\textsuperscript{2}DF satisfies a second-order partial differential equation of Monge-Ampere-type, allowing us to develop a learning pipeline that leverages a novel Monge-Ampere regularization to directly learn S\textsuperscript{2}DF from raw unoriented point clouds \textit{without} supervision from ground-truth S\textsuperscript{2}DF values. Extensive experiments across multiple datasets show that our method significantly outperforms state-of-the-art supervised approaches that require ground-truth surface information as supervision for training. The source code is available at \url{https://github.com/chuanxiang-yang/S2DF}. 
\end{abstract}


\input{Section/introduction}

\input{Section/relatedwork}
\input{Section/method}
\input{Section/experiment}

\input{Section/limitation_conclusion}

\bibliographystyle{IEEEtran}
\bibliography{IEEEabrv,bibliography}

\begin{IEEEbiography}[{\includegraphics[width=1in,height=1.25in,clip,keepaspectratio]{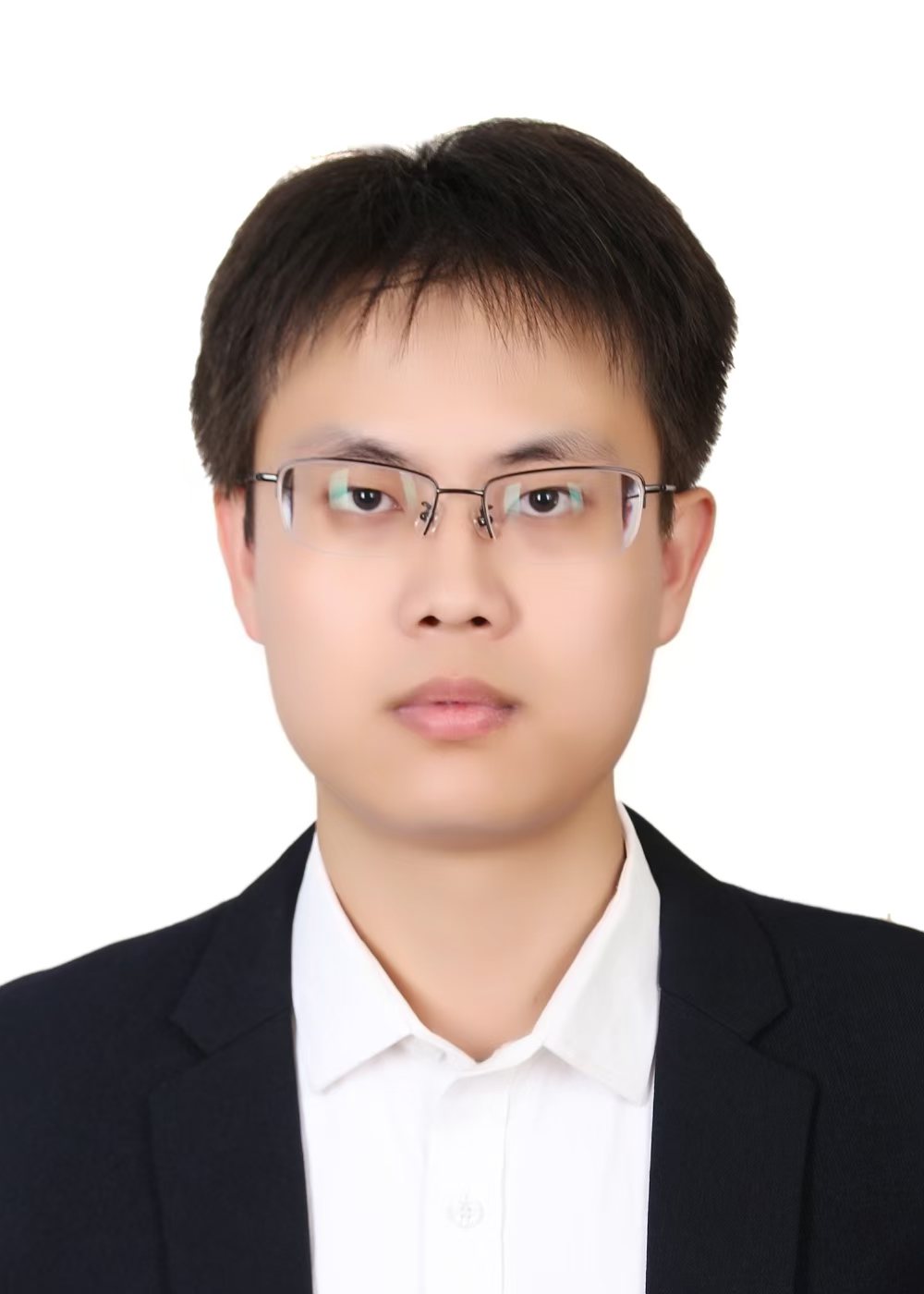}}]{Chuanxiang Yang} received the BE degree from the School of Software, Shandong University, Jinan, in 2023. He is currently working toward the master’s degree in software engineering with the School of Software, Shandong University. His main research interests include computer vision, deep learning and 3D reconstruction.
\end{IEEEbiography}

\begin{IEEEbiography}[{\includegraphics[width=1in,height=1.25in,clip,keepaspectratio]{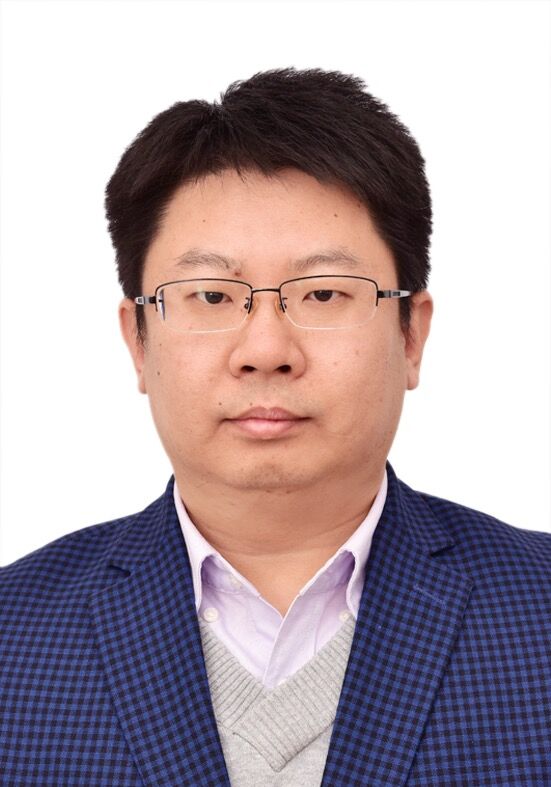}}]{Yuanfeng Zhou} received a master's and Ph.D. from the School of Computer Science and Technology, Shandong University, Jinan, China, in 2005 and 2009, respectively. He held a postdoctoral position with the Graphics Group, Department of Computer Science, The University of Hong Kong, Hong Kong, from 2009 to 2011. He is currently a professor with the School of Software, Shandong University, where he is also the leader of the IGIP Laboratory. His current research interests include geometric modeling, information visualization, and image processing.
\end{IEEEbiography}

\begin{IEEEbiography}[{\includegraphics[width=1in,height=1.25in,clip,keepaspectratio]{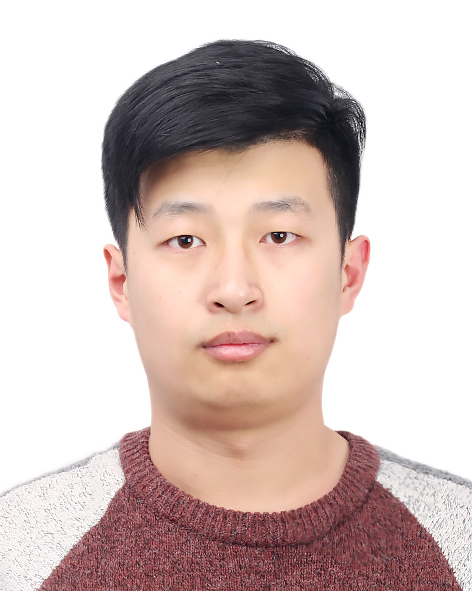}}]{Guangshun Wei} received a  Ph.D. from the School of Software, Shandong University, Jinan, in 2017 and 2022. He held a post-doctoral position with the Graphics Group, Department of Computer Science, The University of Hong Kong, Hong Kong, from 2022 to 2023.  His current research interests include medical image processing and geometric analysis.
\end{IEEEbiography}

\begin{IEEEbiography}[{\includegraphics[width=1in,height=1.25in,clip,keepaspectratio]{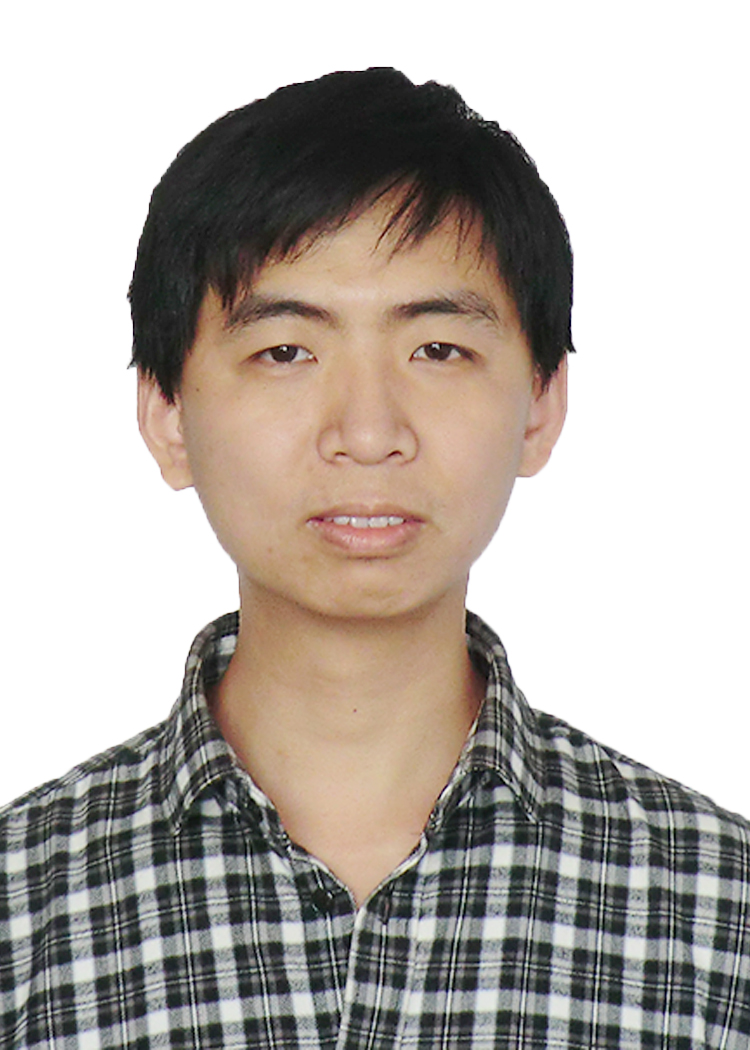}}]{Long Ma} is an associate researcher at the School of Software, Shandong University. He received his Ph.D degree from the School of Computer Science and Technology, Shandong University, in 2017. His research focuses on computer graphics, geometry modeling, and mechanical simulation. 
\end{IEEEbiography}

\begin{IEEEbiography}[{\includegraphics[width=1in,height=1.25in,clip,keepaspectratio]{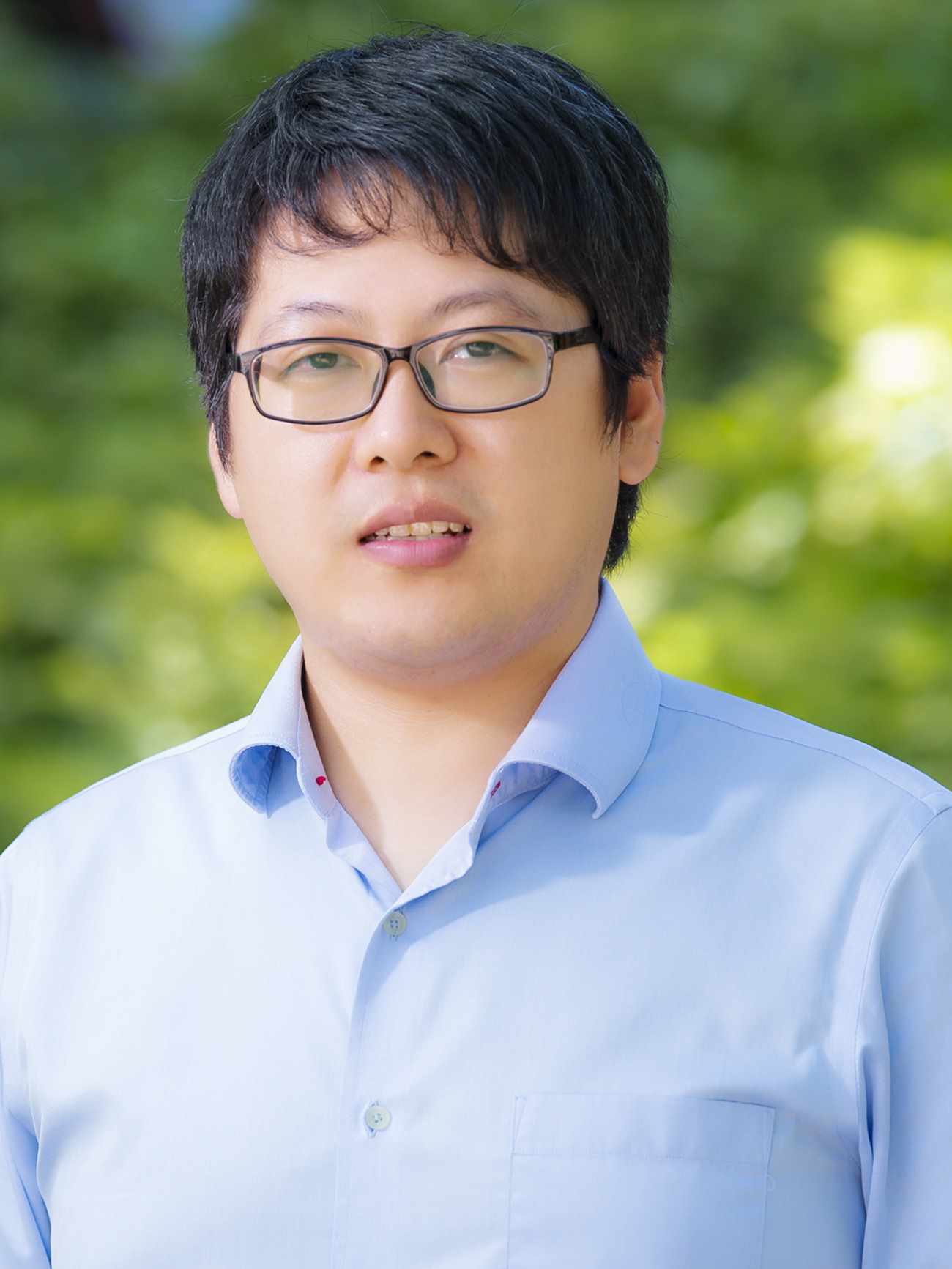}}]{Junhui Hou} (Senior Member, IEEE) is an Associate Professor with the Department of Computer Science, City University of Hong Kong. He holds a B.Eng. degree in information engineering (Talented Students Program) from the South China University of Technology, Guangzhou, China (2009), an M.Eng. degree in signal and information processing from Northwestern Polytechnical University, Xi’an, China (2012), and a Ph.D. degree from the School of Electrical and Electronic Engineering, Nanyang Technological
University, Singapore (2016). His research interests are multi-dimensional visual computing.
Dr. Hou received the Early Career Award (3/381) from the Hong Kong Research Grants Council in 2018 and the NSFC Excellent Young Scientists Fund in 2024. He has served or is serving as an Associate Editor for IEEE Transactions on Visualization and Computer Graphics, IEEE Transactions on Image Processing, IEEE Transactions on Multimedia, and IEEE Transactions
on Circuits and Systems for Video Technology.
\end{IEEEbiography}

\begin{IEEEbiography}[{\includegraphics[width=1in,height=1.25in,clip,keepaspectratio]{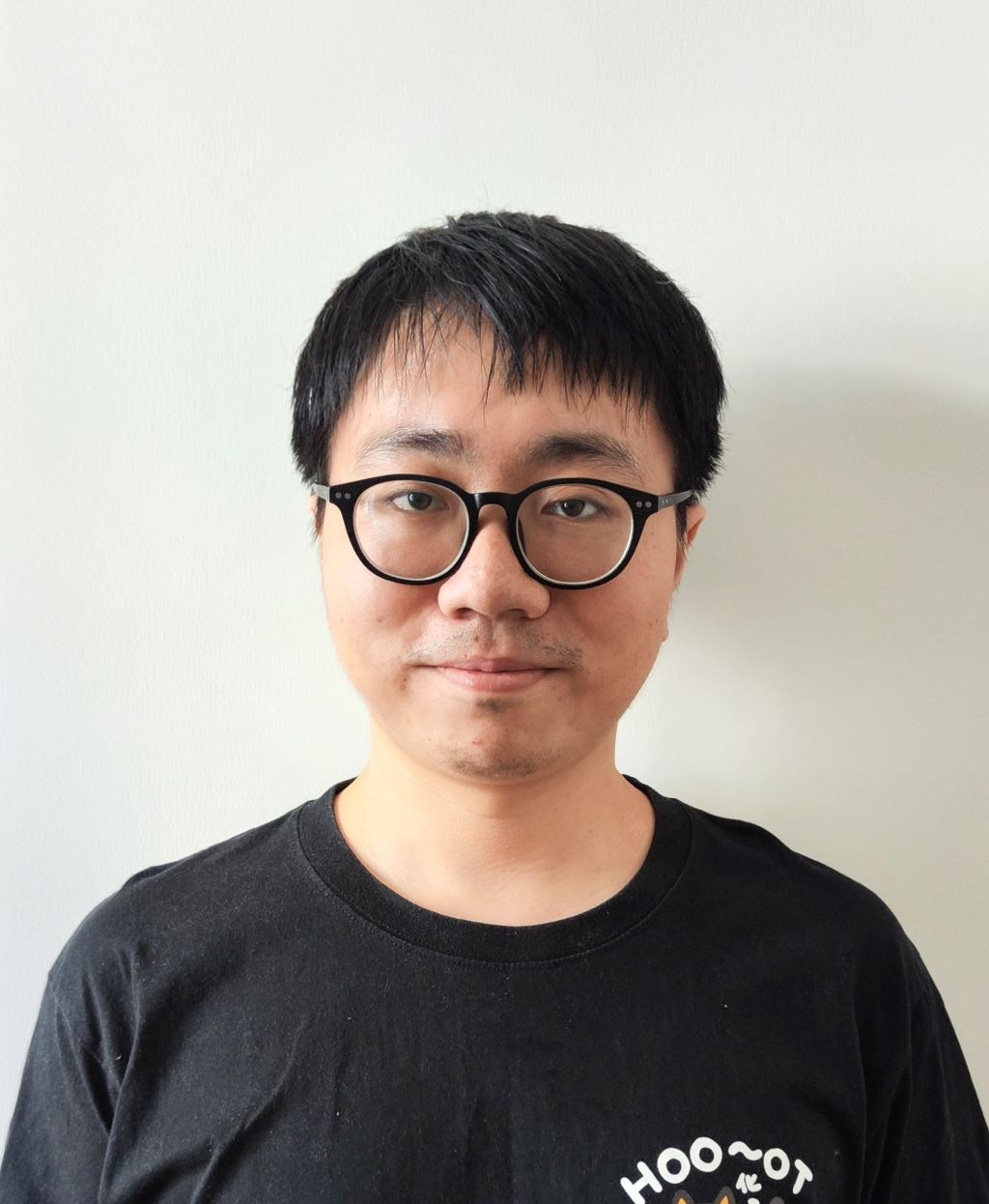}}]{Yuan Liu} is a postdoc researcher at NTU. He received his PhD degree in the University of Hong Kong in 2024. His research mainly concentrates on 3D vision and graphics. He currently works on topics about 3D AIGC including neural rendering, neural representations, and 3D generative models.
\end{IEEEbiography}

\begin{IEEEbiography}[{\includegraphics[width=1in,height=1.25in,clip,keepaspectratio]{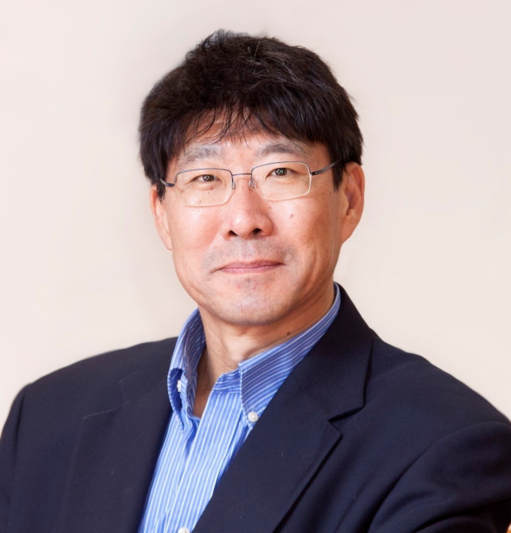}}]{Wenping Wang} (Fellow, IEEE) is a Professor of Computer Science and Engineering at Texas A\&M University. He conducts research in computer graphics, computer vision, scientific visualization, geometric computing, medical image processing, and robotics, and he has published over 400 papers in these fields. He received the John Gregory Memorial Award, Tosiyasu Kunii Award, and Bezier Award for contributions in geometric computing and shape modeling.
\end{IEEEbiography}

\includepdf[pages=-]{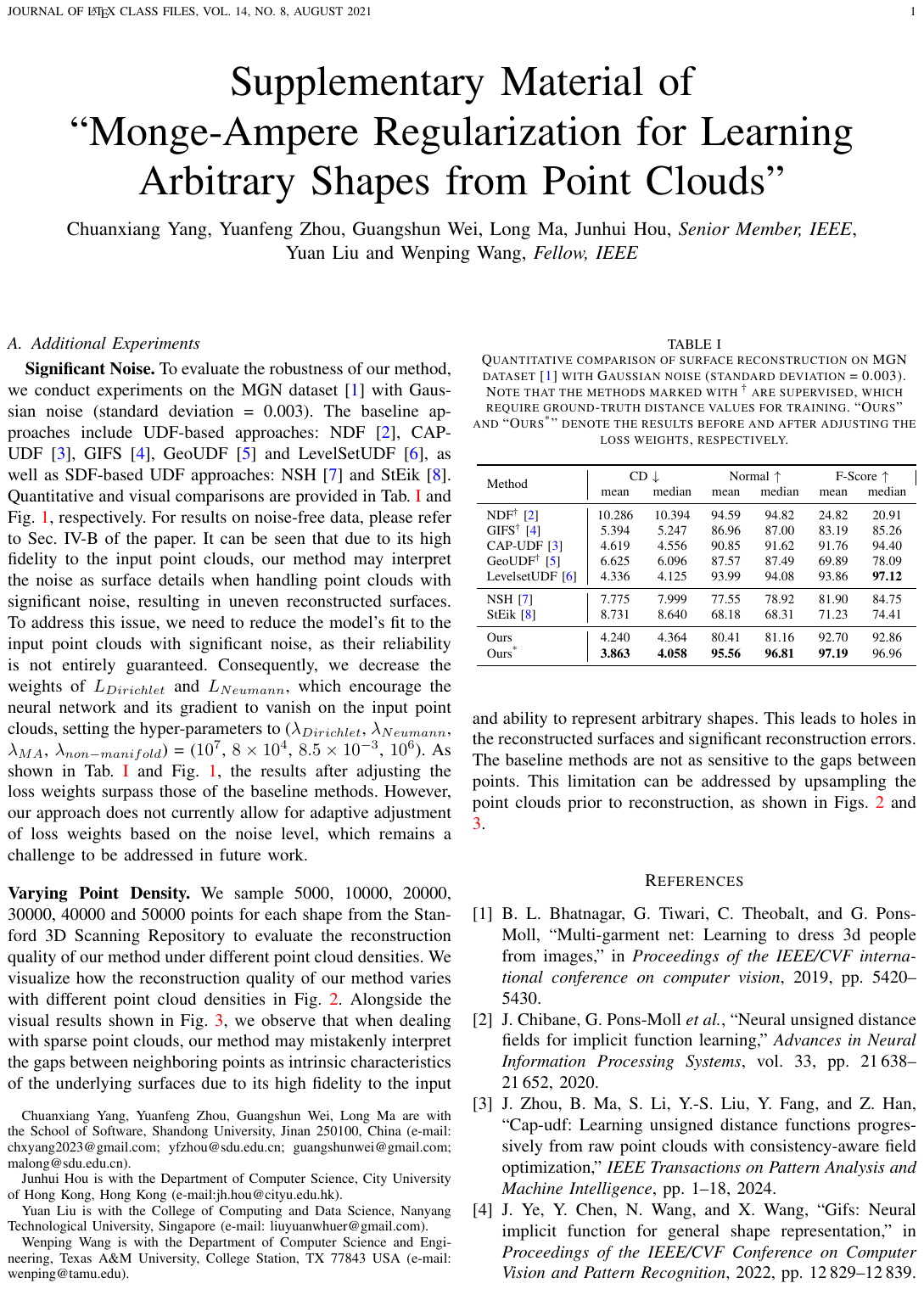}
\end{document}

%% file: Section/introduction.tex
\section{Introduction}
\IEEEPARstart{R}{econstructing} surfaces from discrete point clouds is a fundamental challenge in computer vision and graphics. Over the past few decades, many exceptional algorithms have been developed, with implicit surface reconstruction being a particularly prominent category. With the rapid advancement of deep learning, an increasing number of implicit neural representation methods have achieved impressive performance, which utilize neural networks to model the implicit representations of target surfaces.

The signed distance function (SDF) has become a popular choice for implicit surface representation. Recent works \cite{gropp2020implicit,sitzmann2020implicit}, inspired by the understanding that a viscosity solution to the Eikonal partial differential equation corresponds to the SDF of the target surface, employ Eikonal regularization and boundary condition losses to train neural networks for directly learning SDF from raw point clouds. These methods frame SDF learning as a solution to the Eikonal equations, training a neural network for each point cloud. Due to their simplicity, efficiency, and explainability, this category of methods has undergone continuous refinement through many research efforts \cite{ben2022digs,wang2023neural,koneputugodage2023octree,yang2024stabilizing}.

While methods based on Eikonal regularization have achieved impressive results, their applicability is limited to modeling watertight shapes due to the inherent constraints of SDF. In contrast, the unsigned distance function (UDF) can represent arbitrary shapes since it does not differentiate between interior and exterior regions. However, although UDF, like SDF, maintains a gradient magnitude of 1 in differentiable regions, it does not satisfy the Eikonal equation due to its non-differentiability at the zero-level set. This non-differentiability complicates the application of Eikonal regularization for learning UDF and increases instability in neural networks near the zero-level set during convergence, making it challenging for differentiable neural networks to accurately learn this critical region. To solve these issues, DUDF \cite{fainstein2024dudf} introduces a new implicit representation through the hyperbolic scaling of UDF and defines a new Eikonal problem. However, as an overfitting-based method, DUDF relies on supervision with ground-truth unsigned distances when calculating its newly proposed Eikonal regularization, which limits its practical application.

\begin{figure*}[t]
    \centering
    \includegraphics[width=\textwidth]{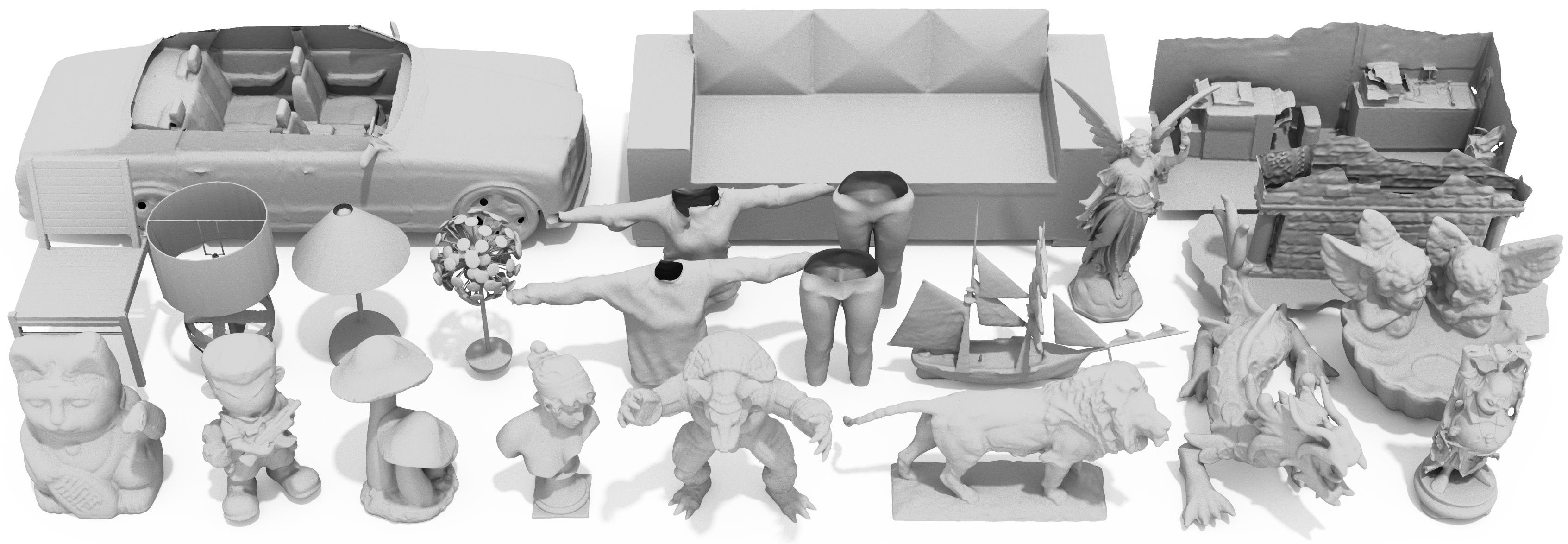}
    \vspace{-8mm}
    \caption{A gallery of surfaces of various types reconstructed from point clouds, including closed surfaces and open surfaces with boundaries, using our method.}
    \label{fig:gallery}
\end{figure*}

\IEEEpubidadjcol

In this paper, we propose the scaled-squared distance function (S\textsuperscript{2}DF) as a novel implicit surface representation to address these challenges. While simple in concept, S\textsuperscript{2}DF does not differentiate between interior and exterior, allowing it to represent arbitrary shapes while remaining differentiable at the zero-level set. Moreover, we observe that the Hessian of S\textsuperscript{2}DF consistently has one identical eigenvalue at arbitrary points in differentiable regions, indicating that it satisfies a Monge-Ampere-type partial differential equation \cite{trudinger2008monge}. Building on this insight, we propose Monge-Ampere regularization, along with appropriate boundary conditions, as the loss function to optimize  neural networks directly learning S\textsuperscript{2}DF from raw unoriented point clouds. Our method avoids the need for ground-truth values at arbitrary points in space; it only requires the ground-truth values of the input point clouds, which are naturally zero. Our method trains a separate neural network for each input point cloud, effectively handling diverse data distributions compared to supervised approaches. We evaluate our method on multiple datasets for the task of surface reconstruction from point clouds. Experimental results demonstrate that our method recovers more precise geometric details and outperforms state-of-the-art approaches, as illustrated in Fig. \ref{fig:gallery}.

In summary, the main contributions of our paper are two-fold:
\begin{enumerate}
\item We introduce S\textsuperscript{2}DF, a novel implicit surface representation capable of representing shapes of arbitrary types while remaining differentiable at the zero-level set.
\item We propose Monge-Ampere regularization to directly learn S\textsuperscript{2}DF from raw unoriented point clouds, eliminating the need for ground-truth S\textsuperscript{2}DF values as supervision during training.
\end{enumerate}

The remainder of this paper is organized as follows. Section \ref{Related Work} reviews implicit surface reconstruction methods from the past few decades. Section \ref{method} introduces our method, including the definition of S\textsuperscript{2}DF, its mathematical properties, and the approach for learning S\textsuperscript{2}DF without supervision
from ground-truth S\textsuperscript{2}DF values. In Section \ref{experiments}, we present experiments to evaluate our method and compare it with state-of-the-art approaches. Section \ref{limitation} discusses the limitations of our method. Finally, in Section \ref{conclusion}, we provide a summary of our method.

%% file: Section/relatedwork.tex
\section{Related Work}
\label{Related Work}
 In this section, we briefly review implicit surface reconstruction methods and related techniques over the past few decades. \revise{We also refer readers to \cite{huang2022surface,berger2017survey,sulzer2023survey} for the comprehensive survey.}

\subsection{Traditional Implicit Approaches}
In recent decades, numerous traditional implicit reconstruction methods \cite{ohtake2005multi,hoppe1992surface,oztireli2009feature} have been extensively explored and applied in surface reconstruction. The moving least-square methods (MLS) \cite{kolluri2008provably,shen2004interpolating,alexa2001point} use the weighted average of signed distances to oriented tangent planes supported at the nearby points to approximate SDF. The methods based on radial basis function \cite{huang2019variational,carr2001reconstruction,li2016sparse,turk2002modelling} leverage the combination of radial basis functions to represent SDF. Poisson Surface Reconstruction \cite{kazhdan2006poisson} and its subsequent variants \cite{kazhdan2013screened,kazhdan2020poisson,bolitho2009parallel} cast fitting occupancy field into solving the Poisson equation. To eliminate the dependency of PSR on normals, iPSR \cite{hou2022iterative} iteratively invokes PSR until convergence and takes as input point samples with normals directly computed from the surface obtained in the preceding iteration in each iteration. Stochastic PSR \cite{sellan2022stochastic} extends PSR with the statistical formalism of the Gaussian Process and represents the reconstructed shape as a modified Gaussian Process instead of outputting an implicit function.

\subsection{Implicit Neural Representations}
With the rapid development of deep learning, neural implicit representations \cite{lin2022surface,huang2022neural,liu2021deep,wang2023neural} have demonstrated significant competitiveness in modeling fine geometric details and complex topologies. Existing neural implicit representation methods typically use occupancy field (OF), signed distance function (SDF), or unsigned distance function (UDF) to represent shapes. We provide individual overviews of these three categories.
\subsubsection{OF-based Implicit Neural Representation}
Onet \cite{mescheder2019occupancy} as the pioneering method based on OF employs an encoder to learn a global latent code for each shape, followed by a decoder that takes the latent code and a query point as input and outputs the occupancy of the query point. ConvOnet \cite{peng2020convolutional} combines convolutional encoders with implicit occupancy decoders to allow for integrating local information and improve the reconstruction quality. POCO \cite{boulch2022poco} leverages point convolution to compute a latent vector for each input point and performs a learning-based interpolation on nearest neighbors using inferred weights for high-fidelity reconstruction. ALTO \cite{wang2023alto} introduces alternating latent topologies to further enhance the quality of the reconstruction.
\subsubsection{SDF-based Implicit Neural Representation}
DeepSDF \cite{park2019deepsdf} optimizes a latent code and outputs the signed distances of query points. LIG \cite{jiang2020local} and DeepLS \cite{chabra2020deep} divide 3D shapes into local regions and encode each one independently. Subsequently, LP-DIF \cite{wang2023lp} trains multiple decoders to further improve reconstruction quality. SAL \cite{atzmon2020sal} and SALD \cite{atzmon2020sald} explore sign agnostic learning to recover SDF from raw data. Neural-Pull \cite{ma2021neural} learns SDF by pulling query 3D locations to their closest points on the surface. IGR \cite{gropp2020implicit} proposes an Eikonal term to regularize neural networks towards the SDF of underlying surfaces. \revise{Incorporating IGR’s loss function, SIREN \cite{sitzmann2020implicit} leverages periodic activation functions and a corresponding initialization scheme to enable neural networks to capture fine details.} DiGS \cite{ben2022digs}, Neural-Singular-Hessian \cite{wang2023neural}, OG-INR \cite{koneputugodage2023octree} and StEik \cite{yang2024stabilizing} respectively propose different constraints to extend SIREN to unoriented point clouds. \revise{PG-SDF \cite{koneputugodage2024small} introduces a point-guided optimization approach to guide the network towards the correct SDF values for surface reconstruction from unoriented point clouds.} While these methods produce impressive reconstruction results, they can only represent watertight objects because they assume that the surface divides the entire space into inside and outside.

\subsubsection{UDF-based Implicit Neural Representation}
 UDF is a more general implicit representation. NDF \cite{chibane2020neural} uses the same shape latent encoding as IF-Nets \cite{chibane2020implicit} and trains a decoder to predict the unsigned distances to the underlying surface for query points. HSDF \cite{wang2022hsdf} simultaneously learns UDF and an additional sign field. Similarly, GIFS \cite{ye2022gifs} learns both UDF and the relationships between every two different points. GeoUDF \cite{ren2023geoudf} introduces a geometry-guided learning method for UDF and gradient estimation, but it requires a considerable amount of GPU memory when dealing with dense point clouds. The methods mentioned above all require UDF ground truth, which is difficult and time-consuming to obtain. And these methods struggle to generalize to beyond the distribution of their training data. Inspired by Neural-Pull, CAP-UDF \cite{zhou2022learning,zhou2024cap} utilizes a field consistency loss to directly learn UDF from raw point clouds. LevelSetUDF \cite{zhou2023learning} proposes to learn a more continuous zero-level set in UDF with level set projections. However, they are unable to effectively capture intricate geometric details, and the reconstructed surfaces sometimes exhibit ghost geometry. UODF \cite{lu2024unsigned} introduces the unsigned orthogonal distance field, while DUDF \cite{fainstein2024dudf} presents a hyperbolic scaling of UDF. However, as overfitting-based methods, both UODF and DUDF require supervision with ground-truth distance values when learning their proposed implicit functions.

%% file: Section/method.tex
\begin{figure}
    \centering
    \begin{overpic}[width=\linewidth]{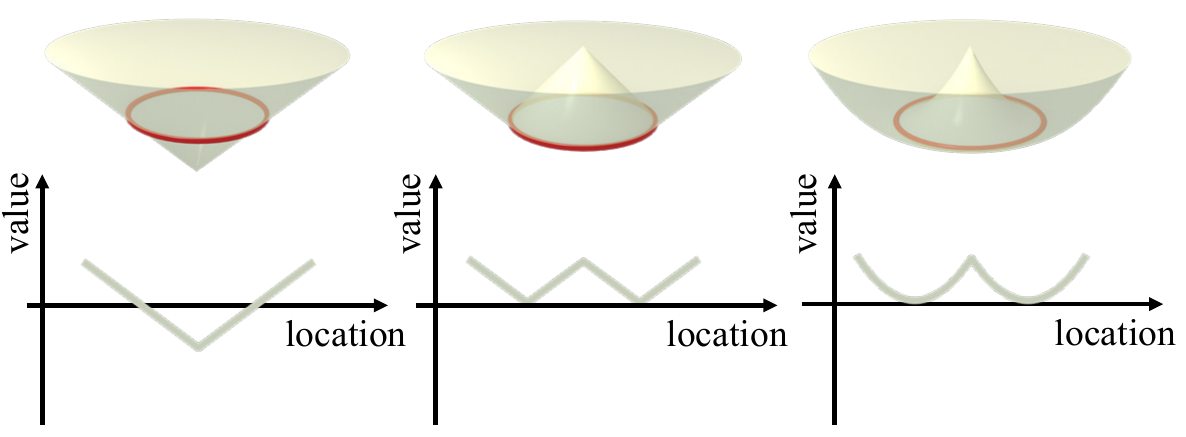}
               \put(13.6, -1.5){\small{SDF}}
               \put(45.8, -1.5){\small{UDF}}
               \put(78, -1.5){\small{S\textsuperscript{2}DF}}
    \end{overpic}
    \caption{Visualization for the SDF, UDF and S\textsuperscript{2}DF of a circle curve. S\textsuperscript{2}DF is introduced to solve the non-differentiability of UDF at the zero-level set while retaining the ability to represent arbitrary shapes.}
    \label{fig:defferent implicit representations}
\end{figure}
\section{Proposed Method}
\label{method}
\textit{Overview}. Implicit surface representation involves utilizing the level set of a function $t(\boldsymbol{x}): \mathbb{R}^3 \rightarrow \mathbb{R}$ to define the surface $\mathcal{S}$:
\begin{equation}
  \mathcal{S} = \{\boldsymbol{x} \in \mathbb{R}^3 \mid t(\boldsymbol{x})=r  \}, 
\end{equation}
where $r$ is 0.5 for OF and $r$ is 0 for SDF and UDF. In particular, implicit neural representation refers to employing a neural network $f(\boldsymbol{x};\theta)$, parameterized by $\theta$, to approximate the function $t(\boldsymbol{x})$.

Currently, SDF and UDF are the most commonly used implicit surface representations. SDF assumes that a surface separates space into interior and exterior regions, restricting it to representing only watertight surfaces. On the contrary, UDF can depict arbitrary surfaces as it does not differentiate between inside and outside; however, it lacks differentiability at the zero-level set. 
Both SDF and UDF uphold a gradient magnitude of 1 within differentiable regions, indicating that their function values exhibit consistent and sufficient variation as the distance from the target surface changes.

Our goal is to create an implicit representation that combines the advantages of both approaches, meeting the following criteria:
\begin{itemize}
    \item generality to represent arbitrary surfaces;
    \item differentiability at the zero-level set; and 
    \item significant variation in implicit function values with increasing distance from the target surface, particularly near the zero-level set. 
\end{itemize}

To achieve this goal, we propose the scaled-squared distance function (S$^2$DF), as illustrated in Fig. \ref{fig:defferent implicit representations}. In what follows, we first introduce the definition of S$^2$DF in Sec. \ref{subsec:definition_S2DF}, followed by its theoretical properties in Sec. \ref{subsec:S2DF_property}. Built upon the properties, we finally propose a learning framework to reconstruct surfaces from 3D point clouds \textit{without} requiring ground-truth S$^2$DF as supervision in Sec. \ref{subsec:surface_pc}.

\subsection{Definition of Scaled-Squared Distance Function}
\label{subsec:definition_S2DF}
We initially introduce the squared distance function to implicitly represent a surface  $\mathcal{S}$, defined as
\begin{equation}
  \mathcal{S} = \{\boldsymbol{x} \in \mathbb{R}^3 \mid t(\boldsymbol{x}) = 0, t(\boldsymbol{x}) = g(\boldsymbol{x})^2 \},
  \label{eq:SSDF1}
\end{equation}
where $g(\boldsymbol{x})$ is the distance function, representing the shortest distance from point $\boldsymbol{x}$ to the surface $\mathcal{S}$. It is evident that like UDF, the above-defined squared distance function does not distinguish between inside and outside, enabling the representation of arbitrary surfaces while maintaining differentiability at the zero-level set.

Moreover, given that points in close proximity to the zero-level set have very small values and their values given that points in close proximity to the zero-level set have very small values and their values given that points in close proximity to the zero-level set have very small values and their values given that points in close proximity to the zero-level set have very small values  
exhibit minimal variations as the distance from the target 
surface alters, distinguishing these points effectively poses a challenge for neural networks during learning.
To address this issue, we adjust the distance function outlined in Eq. \eqref{eq:SSDF1}
by scaling 
it with a constant factor of $K$.  
As illustrated in the inset, 
\setlength{\columnsep}{0.05in}
\begin{wrapfigure}[7]{r}{0.46\linewidth}
  \vspace{-3.8mm}
  \includegraphics[scale = 0.21]{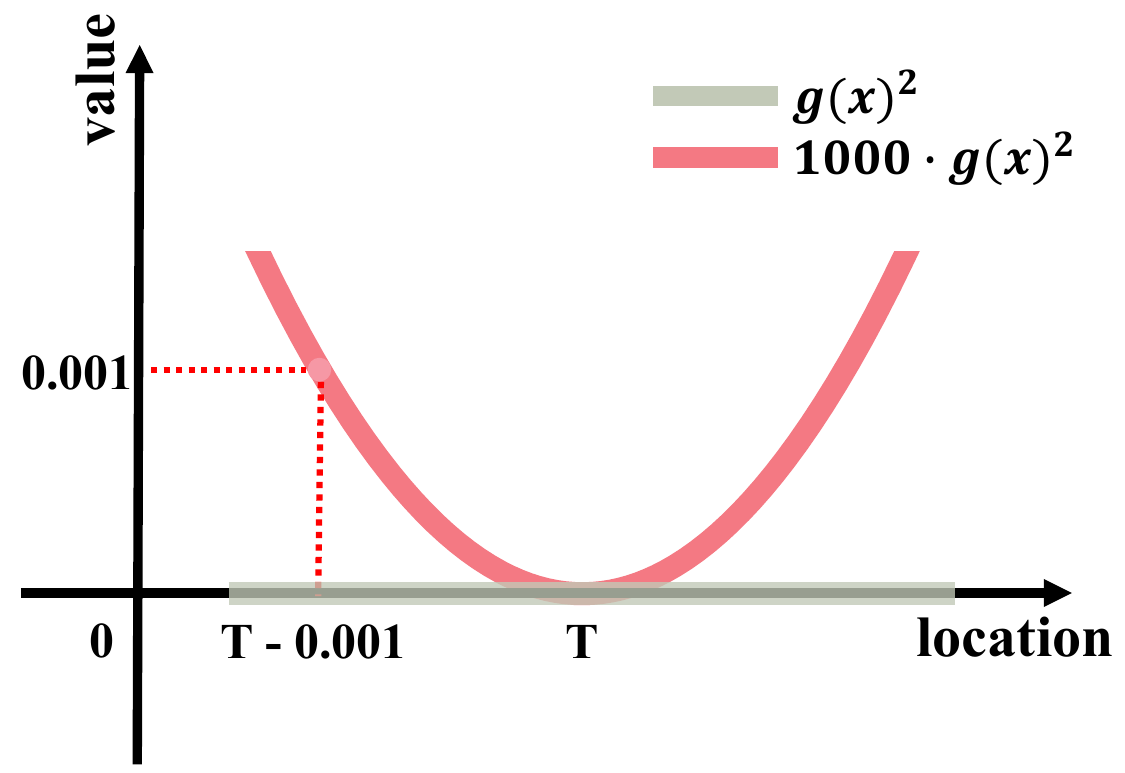}
\end{wrapfigure}multiplying by $K=10^3$ significantly amplifies the difference in implicit function values for any two points near the zero-level set with different distances from it, which can greatly reduce the difficulty of accurately learning the zero-level set for neural networks.
Consequently, the proposed S$^2$DF for implicitly representing an arbitrary 
\begin{equation}
  \mathcal{S} = \{\boldsymbol{x} \in \mathbb{R}^3 \mid t(\boldsymbol{x}) = 0, t(\boldsymbol{x}) = K \cdot g(\boldsymbol{x})^2 \},
\end{equation}
where $K$ is empirically set to $10^3$ in our 3D surface reconstruction experiments. We also refer readers to Sec. \ref{subsec:ablation} for the ablation study of the effect of $K$ on reconstruction accuracy.

\subsection{Theoretical Properties of S$^2$DF}
\label{subsec:S2DF_property}

As a general differentiable implicit surface representation, S\textsuperscript{2}DF exhibits some favorable mathematical properties, demonstrated as follows. 

\begin{theorem}
 Given a surface, if its S\textsuperscript{2}DF is differentiable at point $\boldsymbol{x}$, the Hessian at $\boldsymbol{x}$ will have an eigenvalue of $2K$. If $\boldsymbol{x}$ lies on the zero-level set, the corresponding eigenvector is the surface normal at that point; otherwise, the corresponding eigenvector aligns with the gradient of S\textsuperscript{2}DF at $\boldsymbol{x}$.
 \label{theorem for hessian}
\label{theorem1}
\end{theorem}
\begin{proof}
Consider a watertight surface, its S\textsuperscript{2}DF is written as  
\begin{equation}
t(\boldsymbol{x}) = K \cdot h(\boldsymbol{x})^2,
\label{start}
\end{equation}
where $h(\boldsymbol{x})$ denotes its SDF. According to the chain rule of differentiation, we derive
\begin{equation}
  \nabla t(\boldsymbol{x}) = 2Kh(\boldsymbol{x}) \cdot \nabla h(\boldsymbol{x}).
  \label{equation:first-order1}
\end{equation}
And then, as depicted in Fig. \ref{fig:differential properties}, based on  $\Vert \nabla h(\boldsymbol{x}) \Vert_2 = 1$, we have
\begin{equation}
   \Vert \nabla t(\boldsymbol{x}) \Vert_2^2 = 4K \cdot t(\boldsymbol{x}) \label{equation:first-order2}.
\end{equation}
Differentiating both sides of Eq. \eqref{equation:first-order2}, we further obtain
\begin{equation}
    H_t(\boldsymbol{x}) \nabla t(\boldsymbol{x}) = 2K \cdot \nabla t(\boldsymbol{x}),
    \label{end}
\end{equation}
where $H_t(\boldsymbol{x})$ indicates the Hessian of S\textsuperscript{2}DF. As the gradients of $t(\boldsymbol{x})$ are not equal to zero vector at the non-zero-level set, the Hessian of S\textsuperscript{2}DF always possesses \revise{one eigenvalue equal to $2K$} at the non-zero-level set, with corresponding eigenvector being gradient. 
Based on Eq. \eqref{equation:first-order1} and the chain rule of differentiation, we have
\begin{equation}
    H_t(\boldsymbol{x}) = 2K \cdot \nabla h(\boldsymbol{x}) \cdot \nabla h(\boldsymbol{x})^\top
    + 2Kh(\boldsymbol{x}) \cdot H_h(\boldsymbol{x}).
    \label{start2}
\end{equation}
Furthermore, for a watertight surface, at the zero-level set the Hessian of its S\textsuperscript{2}DF satisfies
\begin{equation}
    H_t(\boldsymbol{x}) = 2K \cdot \nabla h(\boldsymbol{x}) \cdot \nabla h(\boldsymbol{x})^\top.
\end{equation}
Since $\Vert \nabla h(\boldsymbol{x}) \Vert_2 = 1$, the following equation holds:
\begin{equation}
    H_t(\boldsymbol{x}) \cdot \nabla h(\boldsymbol{x}) = 2K \cdot \nabla h(\boldsymbol{x}).
    \label{end2}
\end{equation}
Therefore, Theorem \ref{theorem1} is established for watertight surfaces. 

For non-watertight surfaces, the S\textsuperscript{2}DF is equivalent to the scaled-squared UDF. To complete the proof for points on non-zero-level sets, we replace the SDF with UDF in Eqs. \eqref{start} to \eqref{end}. However, due to the non-differentiability of UDF at the zero-level set, we cannot derive the properties of the S\textsuperscript{2}DF on the zero-level set for non-watertight surfaces as we did in Eqs. \eqref{start2} to \eqref{end2}. To address this, for a non-watertight surface, we assume it is parameterized as $\boldsymbol{r}(u,v)$, and based on this, construct a parameterization of space as follows:
\begin{equation}
   \boldsymbol{x} = \boldsymbol{p}(u,v,d) = \boldsymbol{r}(u,v) + d\boldsymbol{n},
\end{equation}
where $\boldsymbol{n} = \frac{\boldsymbol{r}_u \times \boldsymbol{r}_v}{\Vert \boldsymbol{r}_u \times \boldsymbol{r}_v \Vert_2}$ is the unit normal vector of the surface.
The frame under this parameterization is derived as follows:
\begin{equation}
    \boldsymbol{p}_u = \frac{\partial \boldsymbol{p}}{\partial u} = \boldsymbol{r}_u + t\boldsymbol{n}_u, \  \boldsymbol{p}_v = \frac{\partial \boldsymbol{p}}{\partial v} = \boldsymbol{r}_v + t\boldsymbol{n}_v, \  \boldsymbol{p}_d = \frac{\partial \boldsymbol{p}}{\partial d} = \boldsymbol{n}.
\end{equation}
The S\textsuperscript{2}DF is given by:
\begin{equation}
    t(\boldsymbol{x}) = t(\boldsymbol{p}(u,v,d)) = s(u,v,d) = Kd^2.
\end{equation}
The gradient of the function $t(\boldsymbol{x})$ satisfies
\begin{equation}
\left\{
\begin{array}{l}
     \nabla t(\boldsymbol{x}) \cdot \boldsymbol{p}_u = \frac{\partial s(u,v,d)}{\partial u} = 0   \vspace{1ex} \\  \nabla t(\boldsymbol{x}) \cdot \boldsymbol{p}_v = \frac{\partial s(u,v,d)}{\partial v} = 0  \vspace{1ex} \\ 
     \nabla t(\boldsymbol{x}) \cdot \boldsymbol{p}_d = \frac{\partial s(u,v,d)}{\partial d} = 2Kd.
\end{array}
\right.
\end{equation}
Thus, based on $\boldsymbol{p}_u \cdot \boldsymbol{p}_d = 0$ , $\boldsymbol{p}_v \cdot \boldsymbol{p}_d = 0$ and $\boldsymbol{p}_d = \boldsymbol{n}$, we have 
\begin{align}
    \nabla t(\boldsymbol{x}) = 2Kd \cdot \boldsymbol{n},
\end{align}
Furthermore, the following equation holds:
\begin{align}
    H_t(\boldsymbol{x}) \cdot \boldsymbol{p}_d = \frac{\partial(\nabla t)}{\partial d} = 2K \cdot \boldsymbol{n}.
\end{align}
Since $\boldsymbol{p}_d = \boldsymbol{n}$, we obtain:
\begin{equation}
        H_t(\boldsymbol{x}) \cdot \boldsymbol{n} = 2K \cdot \boldsymbol{n}.
\end{equation}
Therefore, for non-watertight surfaces, the Hessian of the S\textsuperscript{2}DF also has an eigenvalue of $2K$ at the zero-level set, with the corresponding eigenvector being the surface normal.
\end{proof}

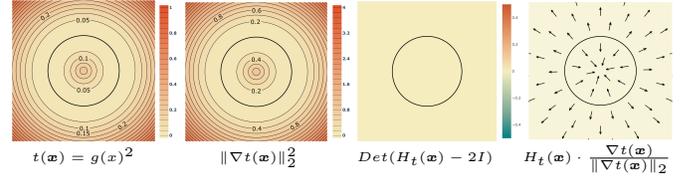
\begin{figure}[t]
    \centering
    \begin{overpic}[width=\linewidth]
    {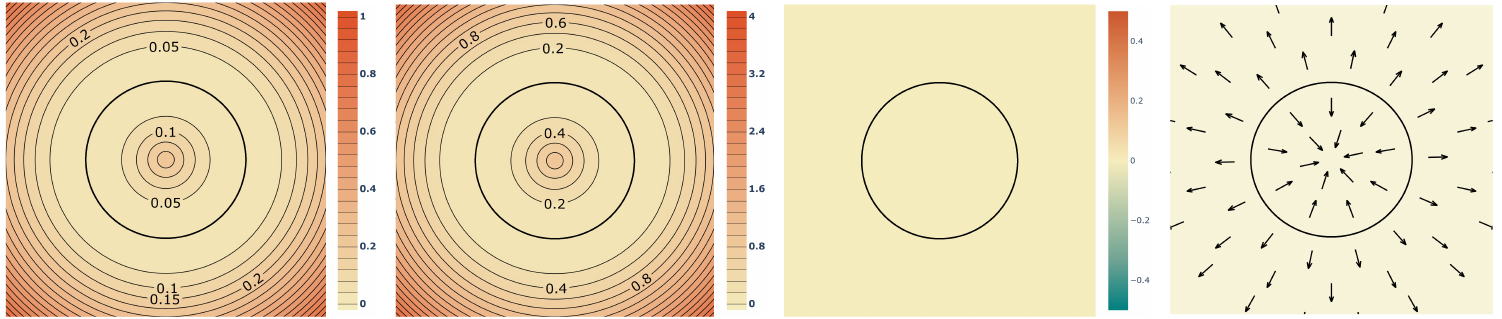}
        \put(2, -3){
        \tiny{$t(\boldsymbol{x}) = g(x)^2$}}
        \put(31.4, -3){\tiny{$\Vert \nabla t(\boldsymbol{x}) \Vert_2^2$}}
        \put(52.1, -3){\tiny{$Det(H_t(\boldsymbol{x})-2I)$}}
        \put(77.0, -3){\tiny{$H_t(\boldsymbol{x}) \cdot \frac{\nabla t(\boldsymbol{x})}{\Vert \nabla t(\boldsymbol{x}) \Vert_2}$}}
    \end{overpic}
    \vspace{-3mm}
    \caption{Visualization of differential properties of S\textsuperscript{2}DF for a circle curve. For S\textsuperscript{2}DF with $K = 1$, the magnitude squared of its gradient is four times itself. Its Hessian always has an eigenvalue of 2 with the corresponding eigenvector being gradient.}
    \label{fig:differential properties}
\end{figure}

\begin{theorem}
 Given a surface $\mathcal{S}$, its S\textsuperscript{2}DF is a weak solution of the following Monge-Ampere partial differential equation:
 \begin{equation}
\left\{
\begin{array}{ll}
{\rm Det}({H}_t(\boldsymbol{x}) - 2K \cdot I) = 0 \\
t(\boldsymbol{x}) = 0,\;  \boldsymbol{x} \in \mathcal{S} \\
\nabla t(\boldsymbol{x}) = \boldsymbol{0},\;  \boldsymbol{x} \in \mathcal{S} \\
\end{array} 
\right.
\label{equation: MA equation}
\end{equation}
where ${\rm Det(\cdot)}$ denotes the determinant, ${H}_t(\boldsymbol{x})$ represents the Hessian of $t(\boldsymbol{x})$ and $I$ is the identity matrix.
\label{theorem for Monge-Ampere}
\end{theorem}
\begin{proof}
According to Theorem \ref{theorem for hessian}, if S\textsuperscript{2}DF is differentiable at a point $\boldsymbol{x}$, the Hessian of S\textsuperscript{2}DF at $\boldsymbol{x}$ must have \revise{an eigenvalue of $2K$}. This means that $2K$ is a root of the characteristic polynomial of the Hessian, and thus the following equation holds: 
\begin{equation}
{\rm Det}({H}_t(\boldsymbol{x}) - 2K \cdot I) = 0.
\end{equation}
Furthermore, if $\boldsymbol{x}$ lies on the surface $\mathcal{S}$, it is evident that both S\textsuperscript{2}DF and its gradient vanish at that point.

\noindent\textbf{Remark 1.} This formulation assumes that $\mathcal{S}$ is a closed surface, allowing for the application of Dirichlet and Neumann boundary conditions. However, for open surfaces, the domain is not well-defined and cannot support boundary constraints. To address this issue, we follow the approach of DUDF \cite{fainstein2024dudf}, reformulating the problem as an initial value problem, where Dirichlet and Neumann conditions prescribe the initial values for $t(\boldsymbol{x})$ and its gradient.

\noindent\textbf{Remark 2.} Theorem \ref{theorem for Monge-Ampere} demonstrates that S\textsuperscript{2}DF is a solution to the Monge-Ampere equation; however, we have yet to theoretically proven whether S\textsuperscript{2}DF is the unique solution. Nevertheless, the method for learning S\textsuperscript{2}DF we develop based on this in Sec. \ref{subsec:surface_pc} has achieved significant success.
\end{proof}

\subsection{Learning S\textsuperscript{2}DF from 3D Point Clouds without Supervision}
\label{subsec:surface_pc}
Based on the theoretical properties of S\textsuperscript{2}DF demonstrated in the preceding subsection, we propose a learning pipeline to directly learn S\textsuperscript{2}DF from raw unoriented point clouds. Specifically, given an unoriented point cloud $\mathcal{P}$, we employ a neural network $f(\boldsymbol{x};\theta)$, parameterized by $\theta$, to approximate the S\textsuperscript{2}DF of the underlying surface $\mathcal{S}$ without relying on ground-truth S\textsuperscript{2}DF values as supervision. That is, $f(\boldsymbol{x};\theta)$ takes a 3D query location $\boldsymbol{x}$ as input and outputs the corresponding S\textsuperscript{2}DF value.

Architecturally, we utilize the sinusoidal representation network (SIREN) \cite{sitzmann2020implicit} to construct $f(\boldsymbol{x};\theta)$. SIREN is a fully connected neural network that uses the sine function as its activation function, excelling at modeling fine details and enabling the computation of first- and second-order derivatives of implicit functions. We optimize $f(\boldsymbol{x};\theta)$ using the following loss function: 
\begin{align}
L = & \lambda_{MA}L_{MA} +  \lambda_{Dirichlet}L_{Dirichlet} 
\notag
\\ &  + \lambda_{Neumann}L_{Neumann} 
\notag
\\ & + \lambda_{non-manifold}L_{non-manifold},
\label{total loss}
\end{align}
where $\lambda_{MA}$, $\lambda_{Dirichlet}$, $\lambda_{Neumann}$ and $\lambda_{non-manifold} > 0$ are hyper-parameters balancing different terms. The term $L_{MA}$, serving as our Monge-Ampere regularization, guides the network $f(\boldsymbol{x};\theta)$ towards a solution of the Monge-Ampere partial differential equation and is defined as follows:
\begin{equation}
L_{MA} = \int_{\mathcal{P} \cup \mathcal{Q}} \left| Det(H_f(\boldsymbol{x}) - 2K \cdot I) \right| d\boldsymbol{x},K=1000,
\end{equation}
where $\mathcal{P}$ represents the input point cloud, $\mathcal{Q}$ is a set of sample points from space, ${\rm Det(\cdot)}$ denotes the determinant, ${H}_f(\boldsymbol{x})$ is the Hessian of $f(\boldsymbol{x};\theta)$ and $I$ is the identity matrix. To avoid excessively large outputs from the network $f(\boldsymbol{x};\theta)$, we normalize all shapes and restrict the sampling of $\mathcal{Q}$ to a narrow region near the input point cloud. Details on the sampling method are provided in the experimental section. 
The terms $L_{Dirichlet}$ and $L_{Neumann}$, corresponding to the Dirichlet and Neumann boundary conditions, respectively, are defined as: 
\begin{equation}
L_{Dirichlet} = \int_{\mathcal{P}} \left| f(\boldsymbol{x};\theta) \right| d\boldsymbol{x}, \end{equation}
\begin{equation}
L_{Neumann} = \int_{\mathcal{P}} \Vert \nabla f(\boldsymbol{x};\theta) \Vert_2 d\boldsymbol{x}.
\end{equation}
The term $L_{non-manifold}$, inspired by SIREN \cite{sitzmann2020implicit}, penalizes off-surface points for being close to the zero,  defined as:
\begin{equation}
L_{non-manifold} = \int_{\mathcal{Q}} exp(-\alpha \left| f(\boldsymbol{x};\theta) \right|) d\boldsymbol{x},\alpha=500.
\end{equation}
Fig. \ref{fig:learned field} visualizes the results of our method on toy 2D shapes, demonstrating that it effectively learns high-quality S\textsuperscript{2}DF for both watertight and non-watertight shapes in an unsupervised fashion. 

\begin{figure}[t]
    \centering
    \begin{overpic}[width=\linewidth]{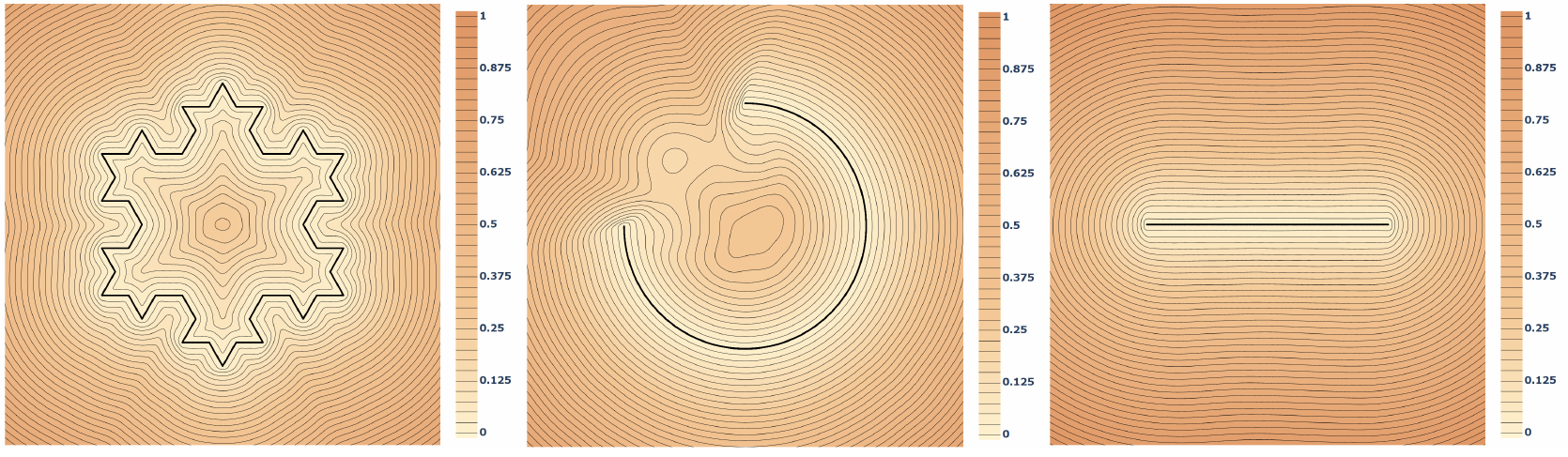}
        \put(9.2, -2){\scriptsize{Snowflake}}
        \put(41.5,-2){\scriptsize{Circular arc}}
        \put(78.2, -2){\scriptsize{Line}}
    \end{overpic}
    \caption{The visualization of the learned implicit functions using our method. Here, we divide the learned  S\textsuperscript{2}DF by K and then take the square root to convert them to UDF. It can be seen that our method can model non-watertight shapes and learn high-quality S\textsuperscript{2}DF.}
    \label{fig:learned field}
\end{figure}

%% file: Section/experiment.tex
\section{Experiments}
\label{experiments}
We evaluate our method on the task of surface reconstruction from point clouds. First, we provide implementation details of our approach. Next, we utilize multiple datasets, including both non-watertight and watertight shapes, to demonstrate that our method can model shapes of arbitrary types and recover high-fidelity geometric details. We conduct ablation studies on the Stanford 3D Scanning Repository to validate the effectiveness of each design in our method. Finally, we compare runtime performance with existing state-of-the-art methods.

\subsection{Experimental Setup}
\textbf{Implementation Details.}
We conduct all experiments on an NVIDIA RTX 3090 GPU with 24GB video memory and an Intel(R) Xeon(R) CPU. The network architecture which is the same as SIREN \cite{sitzmann2020implicit} and consists of 5 hidden layers with 256 units each, is employed to learn S\textsuperscript{2}DF for all datasets. 
\revise{We adopt the same initialization strategy as SIREN. Although our output differs from SIREN's formulation, its proposed initialization method has been validated as effective for various tasks, such as solving the Poisson equation, learning SDF, and fitting 2D images. Our experimental results demonstrate that the initialization strategy is also effective for learning S\textsuperscript{2}DF.} The $K$ is set to 1000 and the loss weights are configured as follows: ($\lambda_{Dirichlet}$, $\lambda_{Neumann}$, $\lambda_{MA}$, $\lambda_{non-manifold}$) = ($10^8$, $8 \times 10^6$, $8.5 \times 10^{-3}$, $10^6$) for non-watertight shapes and ($10^8$, $8 \times 10^6$, $6 \times 10^{-3}$, $10^6$) for watertight shapes. Here, due to the large value of $K$ and the relatively small values of the elements of the network's Hessian at initialization, the absolute values of the main diagonal elements of $(H_f(\boldsymbol{x};\theta) - 2KI)$ are close to $2K$. This means that Monge-Ampere loss and its gradient are very large, so the loss weight of Monge-Ampere is set significantly smaller than the weights of other losses for normalization.

The Adam optimizer \cite{kingma2014adam} is used with the initial learning rate of $3 \times 10^{-4}$ and training is run for 10k iterations. The learning rate decays by a factor of 0.18 at the 4500th, 6000th, 7000th, 8000th, and 9000th iterations. $\mathcal{Q}$ is sampled based on a set of Gaussian distributions, each centered on a point in the input point cloud with a standard deviation equal to 0.01. We sample one point for each distribution. The number of sampled points for $\mathcal{P}$ and $\mathcal{Q}$ at each iteration is set to 15k, which is the same as NSH \cite{wang2023neural} and StEik \cite{yang2024stabilizing}.

\begin{figure*}[t]
    \centering
    \begin{overpic}[width=\linewidth]{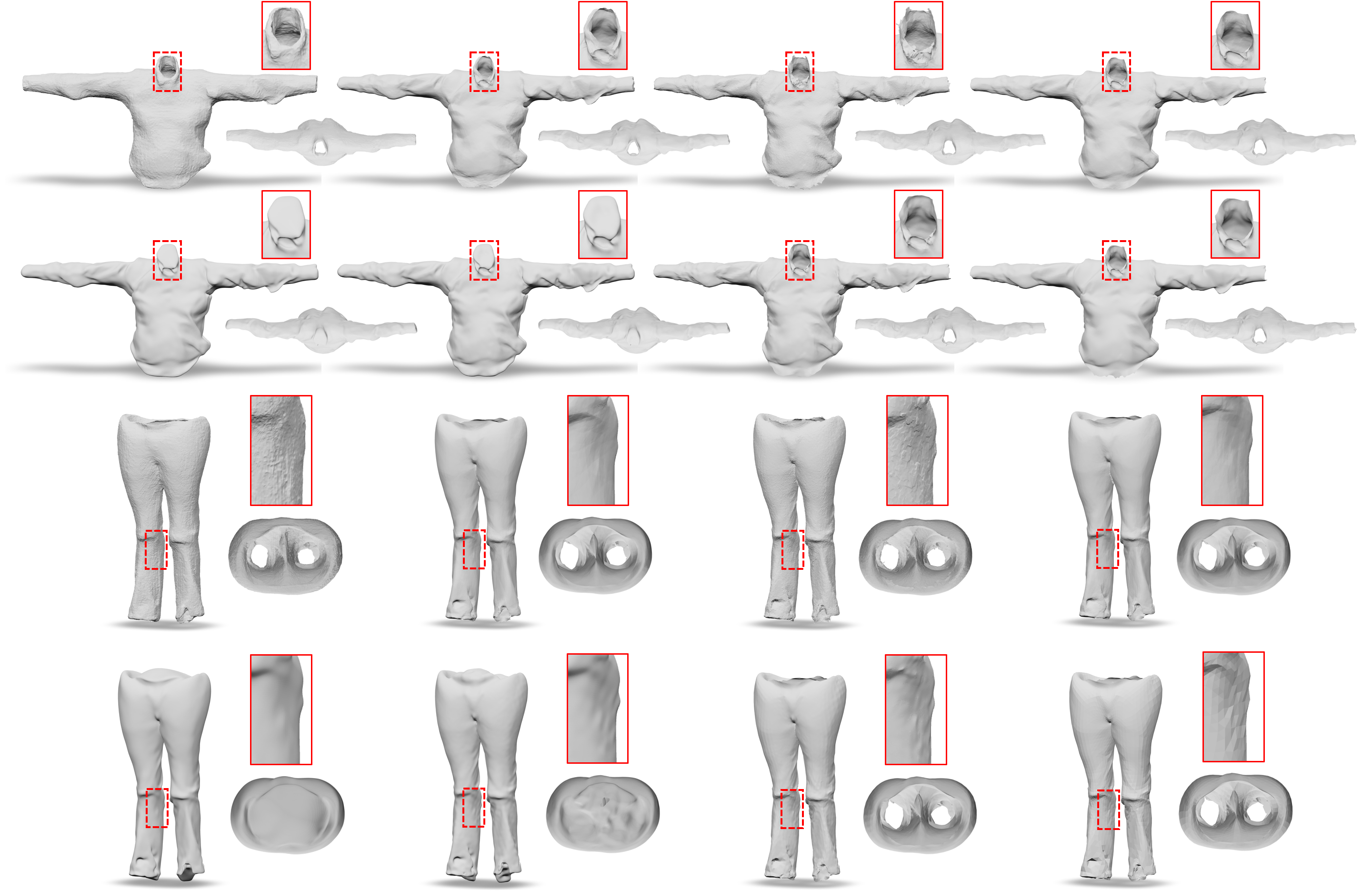}
        \put(10.3, 49.5){\scriptsize{NDF} \cite{chibane2020neural}}
        \put(32, 49.5){\scriptsize{CAP-UDF} \cite{zhou2024cap}}
        \put(55.7,49.5){\scriptsize{GeoUDF} \cite{ren2023geoudf}}
        \put(76.8,49.5){\scriptsize{LevelSetUDF} \cite{zhou2023learning}}
        \put(10.3,36){\scriptsize{NSH} \cite{wang2023neural}}
        \put(33.3,36){\scriptsize{StEik} \cite{yang2024stabilizing}}
        \put(56.75, 36){\scriptsize{Ours}}
        \put(80.6, 36){\scriptsize{GT}}   
        \put(10.3,17.5){\scriptsize{NDF} \cite{chibane2020neural}}
        \put(32.0, 17.5){\scriptsize{CAP-UDF} \cite{zhou2024cap}}
        \put(56.0,17.5){\scriptsize{GeoUDF} \cite{ren2023geoudf}}
        \put(77.0,17.5){\scriptsize{LevelSetUDF} \cite{zhou2023learning}}
        \put(10.3, -1.3){\scriptsize{NSH} \cite{wang2023neural}}
        \put(33.5,-1.3){\scriptsize{StEik} \cite{yang2024stabilizing}}
        \put(57.0, -1.3){\scriptsize{Ours}}
        \put(80.6, -1.3){\scriptsize{GT}}
    \end{overpic}
    \caption{Visual comparison of surface reconstruction on MGN dataset \cite{bhatnagar2019multi}.}
    \label{fig:MGN}
\end{figure*}
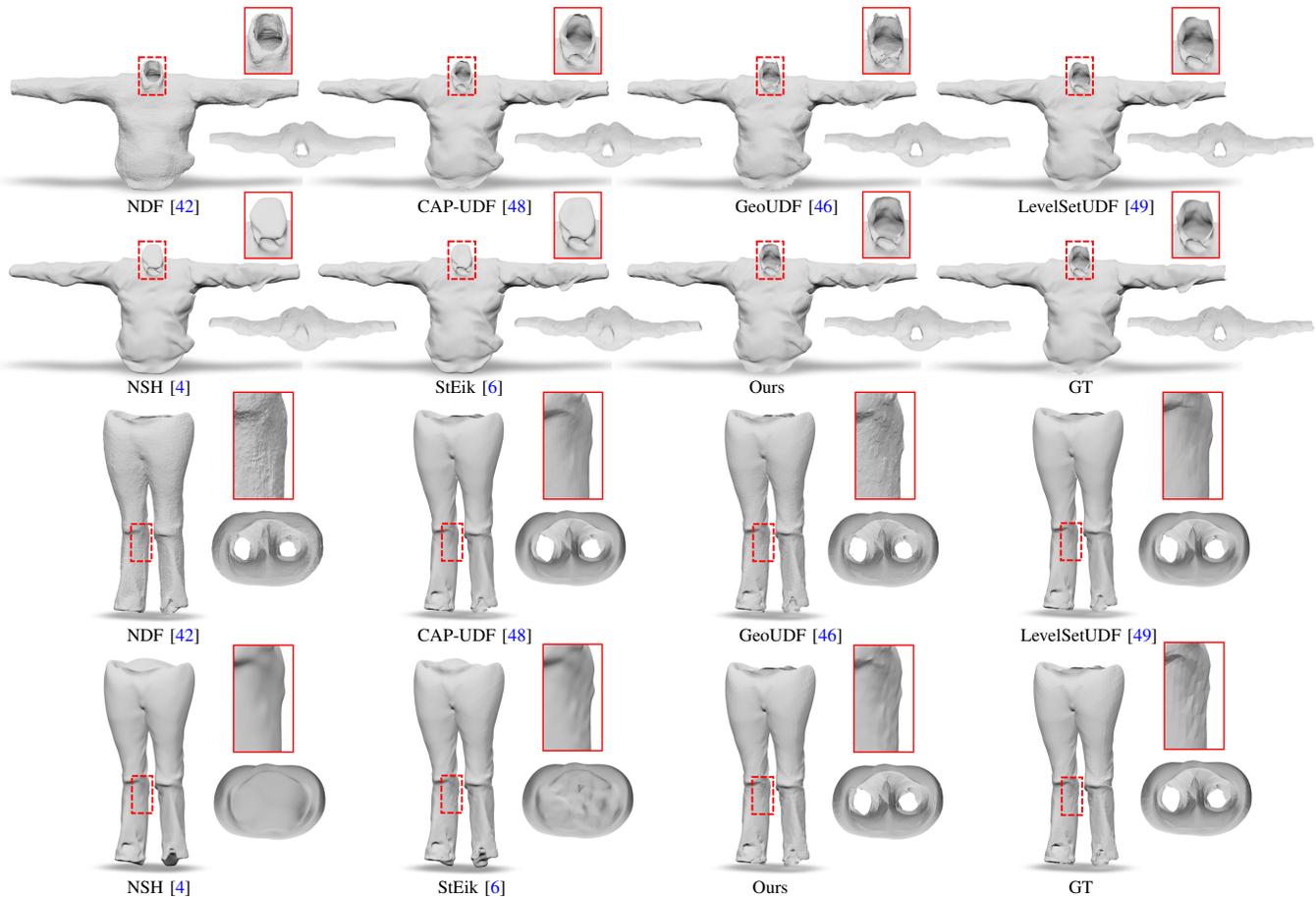

\vspace{1em}
\noindent\textbf{Surface Extraction.}
 We leverage DoubleCoverUDF \cite{hou2023robust} to extract the zero-level set from  S\textsuperscript{2}DF with a grid resolution of $256^3$ and an offset of $5\times10^{-3}$. The number of iterations for DoubleCoverUDF is set to 225 for the first optimization stage and 75 for the second. To ensure a fair comparison, we also apply DoubleCoverUDF, which is able to more accurately capture the correct topological features and decrease the number of non-manifold vertices compared to other methods \cite{hou2023robust}, with the same settings to extract the zero-level set for other UDF-based methods. For SDF-based methods, we generate meshes using the Marching Cubes algorithm \cite{lorensen1998marching} with the same resolution as DoubleCoverUDF.
 \begin{table}[t]
  \caption{Quantitative comparison of surface reconstruction on MGN dataset \cite{bhatnagar2019multi}. Note that the methods marked with \textsuperscript{\dag} are supervised, which require ground-truth distance values for training.}
  \label{table:MGN}
  \centering{
  \resizebox{\linewidth}{!}{
  \begin{tabular}{l|cccccc}
    \toprule
     \multirow{2}{*}{Method} & \multicolumn{2}{c}{CD $\downarrow$} &\multicolumn{2}{c}{NC  $\uparrow$}  &\multicolumn{2}{c}{F-Score $\uparrow$}\\
    &mean&median&mean&median&mean&median\\
        \midrule
    NDF\textsuperscript{\dag} \cite{chibane2020neural}&5.980&6.124&95.76&96.09&79.71&78.15\\
    GIFS\textsuperscript{\dag} \cite{ye2022gifs}&3.790&3.746&90.73&91.08&97.65&97.67\\
    CAP-UDF \cite{zhou2024cap}&2.887&2.839&97.99&97.95&98.87&98.98\\
    GeoUDF\textsuperscript{\dag} \cite{ren2023geoudf}&2.771&2.887&97.52&97.73&99.20&99.20\\
    LevelsetUDF \cite{zhou2023learning}&2.684&2.785&98.47&98.65&99.70&\textbf{99.74}\\
    \midrule
    NSH \cite{wang2023neural}&7.074&6.799&96.60&96.72&93.96&94.03\\
    StEik \cite{yang2024stabilizing}&6.761&6.915&96.08&96.17&93.84&93.86\\
    \midrule
    Ours &\textbf{2.657}&\textbf{2.759}&\textbf{98.94}&\textbf{99.06}&\textbf{99.95}&99.69\\
    \bottomrule
\end{tabular}}}
\end{table}
\vspace{1em}
\noindent\textbf{Evaluation Metrics.}
To compare the reconstruction quality, we sample 100k points from the reconstructed surfaces and use the Chamfer distance (CD), Normal consistency (NC), and F-Score as evaluation metrics. The reported CD uses L1-norm and is scaled by $10^3$ to measure the similarity between the two surfaces. The Normal consistency, expressed as a percentage, refers to the mean absolute cosine of normals in one mesh and normals at nearest neighbors in the other mesh. The F-Score, expressed as a percentage, indicates the harmonic mean of precision and recall, with a default threshold of 0.008.

\subsection{Surface Reconstruction of Non-watertight Shapes}
 Our comparison of the proposed method spans synthetic shapes including MGN \cite{bhatnagar2019multi} and ShapeNet \cite{chang2015shapenet}, as well as real scans from 3D Scene \cite{zhou2013dense} and Waymo \cite{sun2020scalability}. The baseline methods are grouped into three categories of state-of-the-art surface reconstruction techniques:
\begin{itemize}
    \item Supervised UDF Approaches: This category includes NDF \cite{chibane2020neural}, GIFS \cite{ye2022gifs} and GeoUDF \cite{ren2023geoudf}, which rely on ground-truth data to train a generalized model.
    \item Overfitting-based UDF Approaches: This category encompasses CAP-UDF \cite{zhou2024cap} and LevelSetUDF \cite{zhou2023learning}, where neural networks are trained individually for each shape.
    \item Overfitting-based SDF Approaches: This category features NSH \cite{wang2023neural} and StEik \cite{yang2024stabilizing}, which use Eikonal term and other contrains to learn SDF from raw unoriented point clouds.
\end{itemize}

\vspace{1em}
\noindent\textbf{MGN.} The MGN dataset \cite{bhatnagar2019multi} comprises a substantial number of open clothing models. We randomly select 20 shapes from this dataset to evaluate our method, using 20k sampled points for each model as input. The quantitative comparisons are reported in Tab. \ref{table:MGN}, while the visual comparisons are presented in Fig. \ref{fig:MGN}. It is evident that NSH \cite{wang2023neural} and StEik \cite{yang2024stabilizing}, both of which are SDF-based methods using Eikonal term, forcibly close the non-watertight clothing, resulting in significant reconstruction errors. NDF \cite{chibane2020neural} and GeoUDF \cite{ren2023geoudf} can handle open surfaces, but their reconstructed surfaces are rough. In contrast, our method recovers finer geometric details, yielding higher-quality reconstructions. 

\begin{figure*}[t]
    \centering
     \begin{overpic}[width=\textwidth]{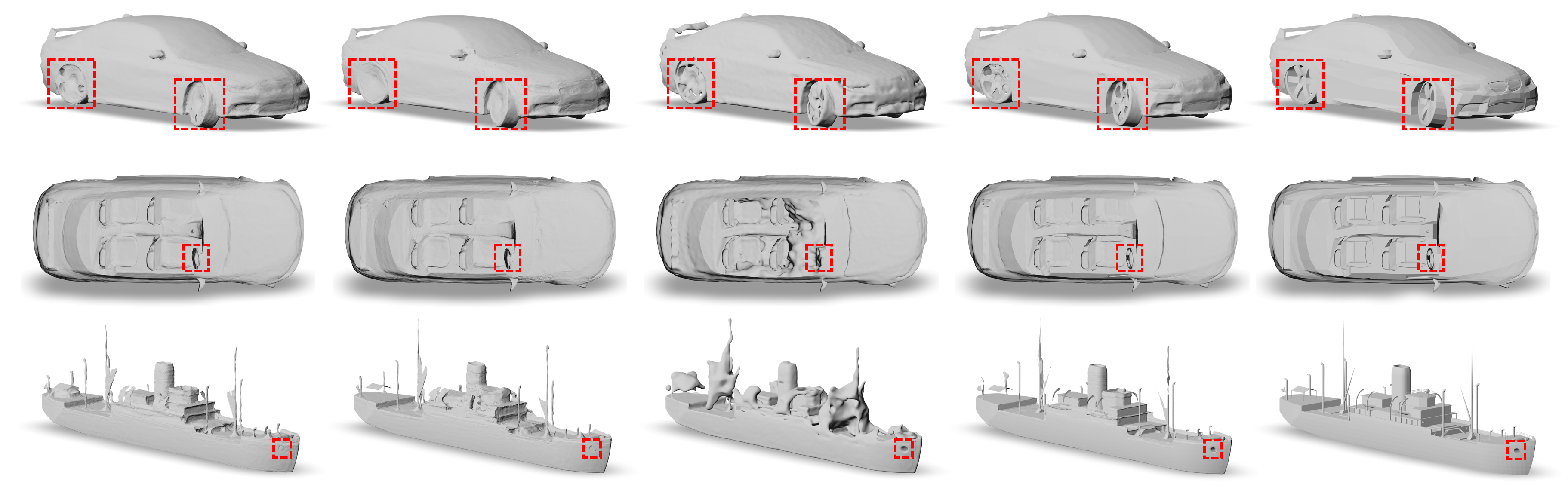}
                \put(6, -1.5){\small{CAP-UDF} \cite{zhou2024cap}}
                \put(24.5, -1.5){\small{LevelSetUDF} \cite{zhou2023learning}}
                \put(47.8, -1.5){\small{StEik} \cite{yang2024stabilizing}}
                \put(68, -1.5){\small{Ours}}
                \put(87, -1.5){\small{GT}}
    \end{overpic}
    \caption{Visual comparison of surface reconstruction under ShapeNet dataset \cite{chang2015shapenet}.}
    \label{fig:shapenet}
\end{figure*}
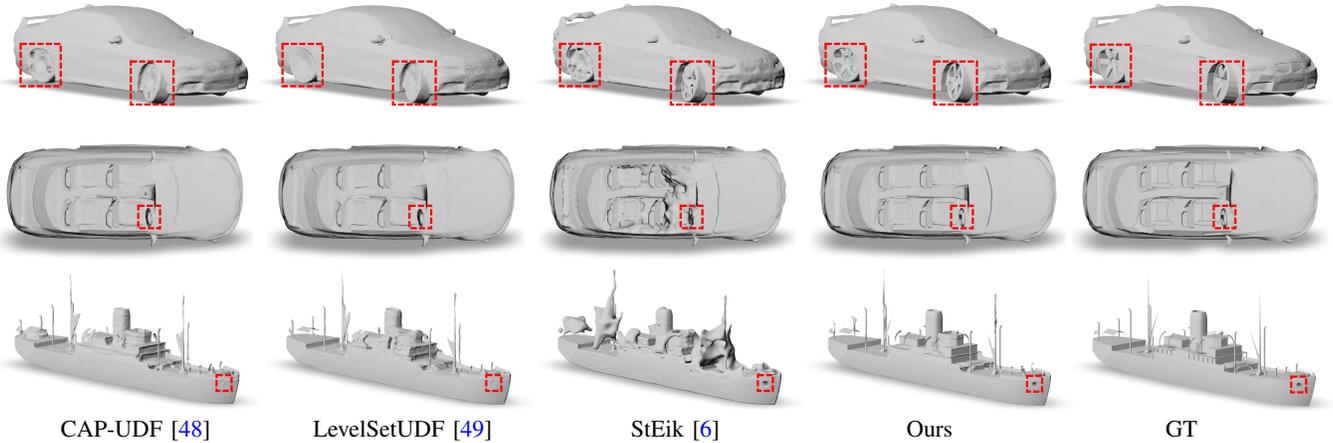

\vspace{1em}

\noindent\textbf{ShapeNet.}
\label{section-shapenet}
 The ShapeNet dataset \cite{chang2015shapenet} features a diverse collection of 3D CAD models, which are non-watertight and exhibit complex internal and thin-layered structures. We evaluate our method against the baseline approaches across five categories: cars, lamps, chairs, sofas and watercraft, with each category containing 20 objects. For each object, we sample 100k points as input. The quantitative results in Tab. \ref{table:shapenet} and the visual comparisons in Fig. \ref{fig:shapenet} demonstrate our method outperforms the baseline approaches both quantitatively and qualitatively. StEik \cite{yang2024stabilizing}, which is based on SDF, fails to effectively deal with these complex objects. CAP-UDF \cite{zhou2024cap} and LevelSetUDF \cite{zhou2023learning} struggle to recover intricate geometric structures such as the tires and steering wheels of cars. In contrast, our method effectively reconstructs these complex objects, yielding more accurate topologies.

\begin{table}[t]
  \caption{Quantitative comparison of surface reconstruction on ShapeNet dataset \cite{chang2015shapenet}. Note that the methods marked with \textsuperscript{\dag} are supervised, which require ground-truth distance values for training.}
  \label{table:shapenet}
  \centering{
  \resizebox{\linewidth}{!}{
  \begin{tabular}{l|cccccc}
    \toprule
    \multirow{2}{*}{Method} & \multicolumn{2}{c}{CD $\downarrow$} &\multicolumn{2}{c}{NC  $\uparrow$}  &\multicolumn{2}{c}{F-Score $\uparrow$}\\
    &mean&median&mean&median&mean&median\\
    \midrule
    NDF\textsuperscript{\dag} \cite{chibane2020neural}&8.631&8.588&90.96&91.78&49.69&48.57\\
    GIFS\textsuperscript{\dag} \cite{ye2022gifs}&5.242&5.156&88.55&88.59&86.28&90.30\\
    CAP-UDF \cite{zhou2024cap}&4.345&4.359&94.20&95.27&89.48&93.29\\
    LevelsetUDF \cite{zhou2023learning}&4.602&4.327&94.10&95.40&88.78&92.40\\
    \midrule
    NSH \cite{wang2023neural}&7.834&5.384&92.45&93.96&83.706&87.54\\
    StEik \cite{yang2024stabilizing}&7.578&6.058&92.47&92.72&83.14&86.01\\
    \midrule
    Ours &\textbf{4.051}&\textbf{4.103}&\textbf{95.58}&\textbf{96.62}&\textbf{91.67}&\textbf{94.92}\\
    \bottomrule
\end{tabular}}}
\end{table}

\begin{table*}[t]
  \caption{Quantitative comparison of surface reconstruction on 3D Scene dataset \cite{zhou2013dense}. Note that the methods marked with \textsuperscript{\dag} are supervised, which require ground-truth distance values for training.}
  \label{table:3D Scene}
  \resizebox{\linewidth}{!}{
  \begin{tabular}{l|cccccc|cccccc}
    \toprule
    \multirow{3}{*}{Method} &\multicolumn{6}{c|}{500/$m^2$}&\multicolumn{6}{c}{1000/$m^2$}\\
    \cmidrule{2-13}
    & \multicolumn{2}{c}{CD $\downarrow$} &\multicolumn{2}{c}{Normal $\uparrow$}  &\multicolumn{2}{c|}{F-Score $\uparrow$} &  \multicolumn{2}{c}{CD $\downarrow$} &\multicolumn{2}{c}{Normal $\uparrow$}  &\multicolumn{2}{c}{F-Score $\uparrow$}\\
    &mean&median&mean&median&mean&median&mean&median&mean&median&mean&median\\
    \midrule
    NDF\textsuperscript{\dag} \cite{chibane2020neural}&7.863&7.971&90.35&90.58&60.53&61.79&9.526&9.614&89.67&90.05&33.80&34.22\\
    GIFS\textsuperscript{\dag} \cite{ye2022gifs}&4.871&4.784&86.97&86.95&91.25&92.96&5.056&4.971&86.15&87.02&89.98&91.18\\
    CAP-UDF \cite{zhou2024cap}&3.690&3.575&91.88&90.13&95.66&96.95&3.769&3.803&91.95&90.99&94.99&95.19\\
    LevelsetUDF \cite{zhou2023learning}&3.746&3.689&92.13&91.11&95.03&96.02&3.681&3.508&92.16&90.61&95.12&\textbf{97.17}\\
    \midrule
    NSH \cite{wang2023neural}&27.23&16.41&84.06&82.45&66.17&66.71&26.49&23.17&83.98&82.70&65.51&65.56\\
    StEik \cite{yang2024stabilizing}&21.25&18.19&84.81&84.84&69.08&68.25&21.53&18.73&84.85&84.85&70.08&70.39\\
    \midrule
    Ours &\textbf{3.624}&\textbf{3.515}&\textbf{93.33}&\textbf{92.87}&\textbf{95.92}&\textbf{97.05}&\textbf{3.570}&\textbf{3.487}&\textbf{93.63}&\textbf{93.41}&\textbf{95.91}&97.15\\
    \bottomrule
\end{tabular}}
\end{table*}
\begin{figure*}[t]
    \centering
    \begin{overpic}[width=\textwidth]{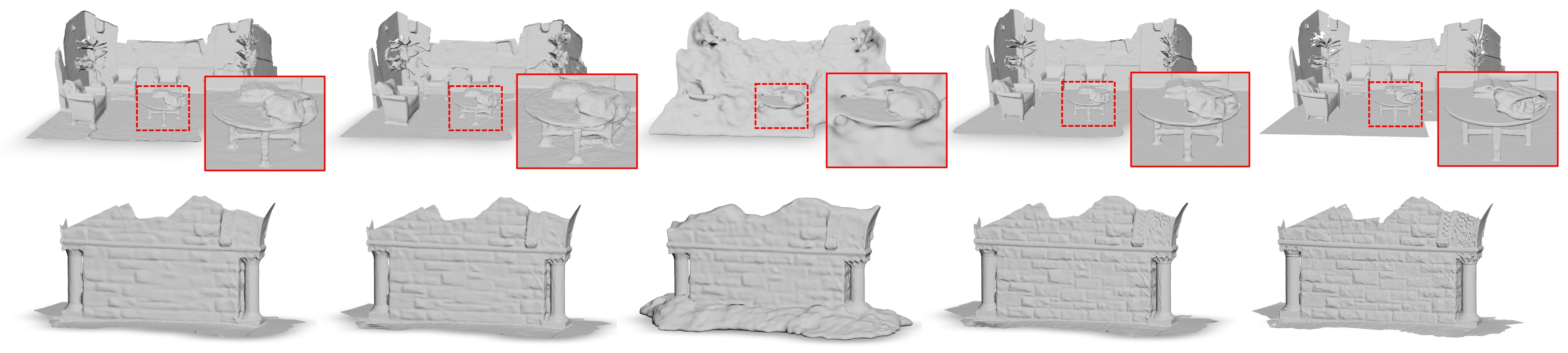}
               \put(6.65, -1.35){\small{CAP-UDF} \cite{zhou2024cap}}
                \put(25.0, -1.35){\small{LevelSetUDF} \cite{zhou2023learning}}
                \put(47, -1.35){\small{StEik} \cite{yang2024stabilizing}}
                \put(66.7, -1.35){\small{Ours}}
                \put(86.5, -1.35){\small{GT}}
    \end{overpic}
    \caption{Visual comparison of surface reconstruction under 3D Scene dataset \cite{zhou2013dense}. Input contains 1K points / $m^2$.}
    \label{fig:3dsence}
\end{figure*}
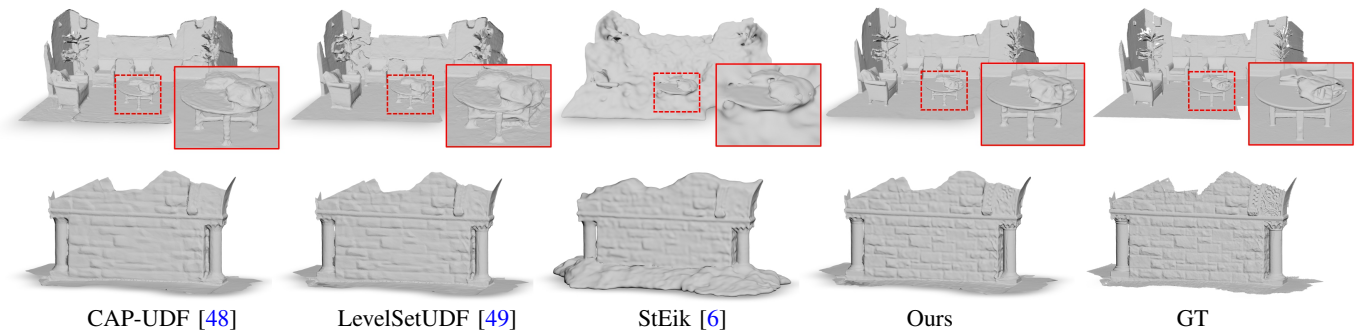

\vspace{1em}
\noindent\textbf{3D Scene.}
We evaluate our method on real scans to further demonstrate its advantages. Following LevelSetUDF \cite{zhou2023learning}, we conduct experiments on the 3D Scene dataset \cite{zhou2013dense}, which comprises real-world scenes with complex topologies and noises. For each scene, we sample 500 and 1000 points per square meter. 
As shown in Tab. \ref{table:3D Scene} and Fig. \ref{fig:3dsence}, our method exhibits clear advantages, achieving superior reconstruction results across both point cloud densities. It effectively adapts to various models and preserves intricate geometric details. In contrast, the supervised methods NDF \cite{chibane2020neural} and GIFS \cite{ye2022gifs} struggle to generalize to this unseen data, while CAP-UDF \cite{zhou2024cap} and LevelSetUDF \cite{zhou2023learning} produce overly smooth surfaces, lacking fine geometric details such as stone patterns on walls. StEik \cite{yang2024stabilizing} also encounters challenges with real-world scene data.

\vspace{1em}
\noindent\textbf{Waymo.}
To comprehensively demonstrate the capabilities of our method, we also use LiDAR scans of the Waymo dataset \cite{sun2020scalability} to conduct experiments on extremely large scenes. We leverage the slide-window strategy to divide the full scene into multiple blocks and reconstruct each one independently. Then we combine the reconstructed local scenes to obtain the complete scene. As illustrated in Fig. \ref{fig:waymo}, our method generates visually more appealing results compared to existing state-of-the-art approaches, even when processing LiDAR scans with various artifacts, such as significant noise, non-uniform distribution, and missing points.

\begin{figure*}[ht]
    \centering
    \begin{overpic}[width=\linewidth]{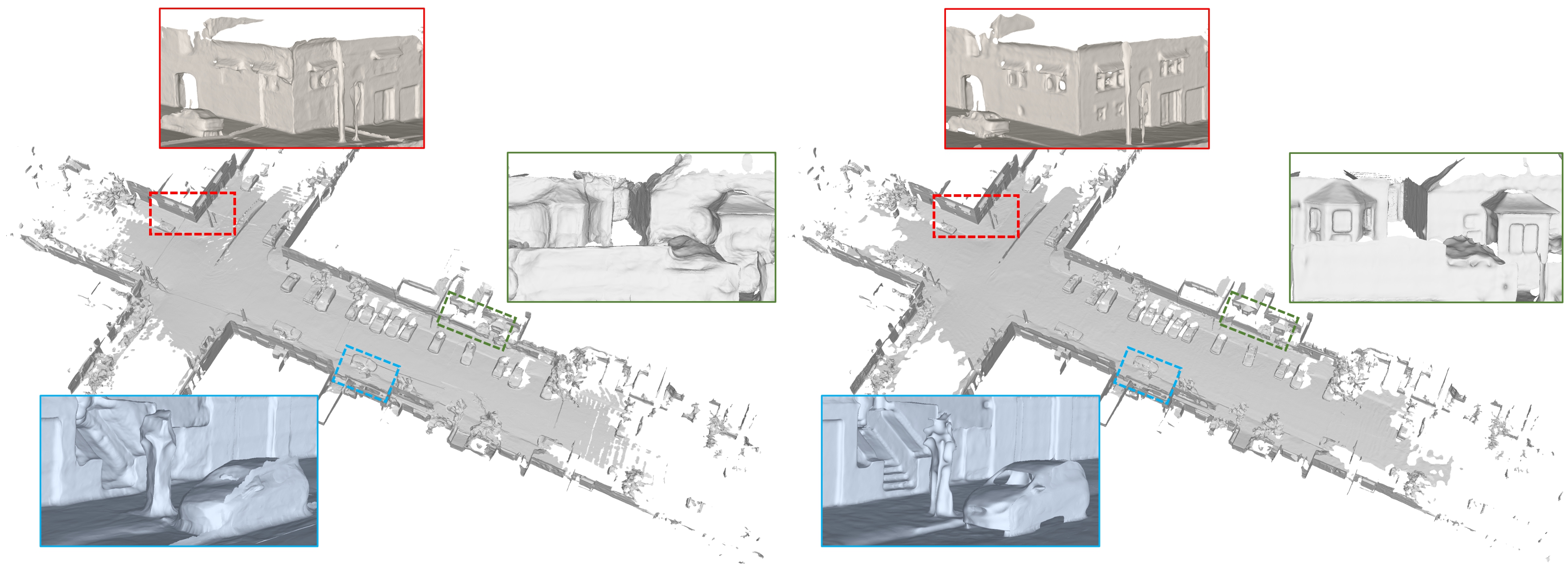}
        \put(20.0, -1.35){\small{LevelSetUDF} \cite{zhou2023learning}}
        \put(73.2, -1.35){\small{Ours}}
    \end{overpic}
    \caption{Visual comparison of surface reconstruction under Waymo dataset \cite{sun2020scalability}.}
    \label{fig:waymo}
\end{figure*}

\subsection{Surface Reconstruction of Watertight Shapes}
To fully show the effectiveness of our approach, we also conduct experiments on watertight objects for comparison with SDF-based methods. We utilize three datasets: ShapeNet \cite{chang2015shapenet}, Thingi10k \cite{zhou2016thingi10k}, and the Stanford 3D Scanning Repository to evaluate our method against baseline approaches. The baseline approaches include overfitting-based SDF approaches: IGR \cite{gropp2020implicit}, SIREN \cite{sitzmann2020implicit}, DiGS \cite{ben2022digs}, NSH \cite{wang2023neural}, StEik \cite{yang2024stabilizing} and PG-SDF \cite{koneputugodage2024small}, as well as overfitting-based UDF approaches: CAP-UDF \cite{zhou2024cap} and LevelSetUDF \cite{zhou2023learning}.

\vspace{1em}
\noindent\textbf{ShapeNet.}
To evaluate our method on watertight shapes, we first preprocess the ShapeNet dataset \cite{chang2015shapenet} using the DISN method \cite{xu2019disn} to remove internal structures and make shapes watertight. We maintain the same split as in the section \ref{section-shapenet}. Note that none of the methods is provided with normals. The quantitative comparison in Tab. \ref{table:ShapeNet processed by DISN} and the visual comparison in Fig. \ref{fig:shapnet process by DISN} demonstrate that our approach effectively reconstructs the lamp with numerous thin rods and achieves higher quality reconstruction results than other methods. \revise{Although PG-SDF \cite{koneputugodage2024small} also successfully reconstructs the lamp and off-road vehicle, the reconstructed surfaces are not smooth.}
We attribute our success to our Neumann loss $L_{Neumann}$. For SDF-based approaches using the Eikonal term, their Neumann loss requires the supervision of point cloud normals. While DiGS \cite{ben2022digs}, NSH \cite{wang2023neural} and StEik \cite{yang2024stabilizing} have explored alternative losses to replace this normal supervision, their performance falls short compared to methods that use normal supervision. In contrast, our $L_{Neumann}$ does not require normal supervision because the gradients of S\textsuperscript{2}DF at zero-level set are $\boldsymbol{0}$. 

\begin{figure*}[ht]
    \centering
    \begin{overpic}[width=\textwidth]{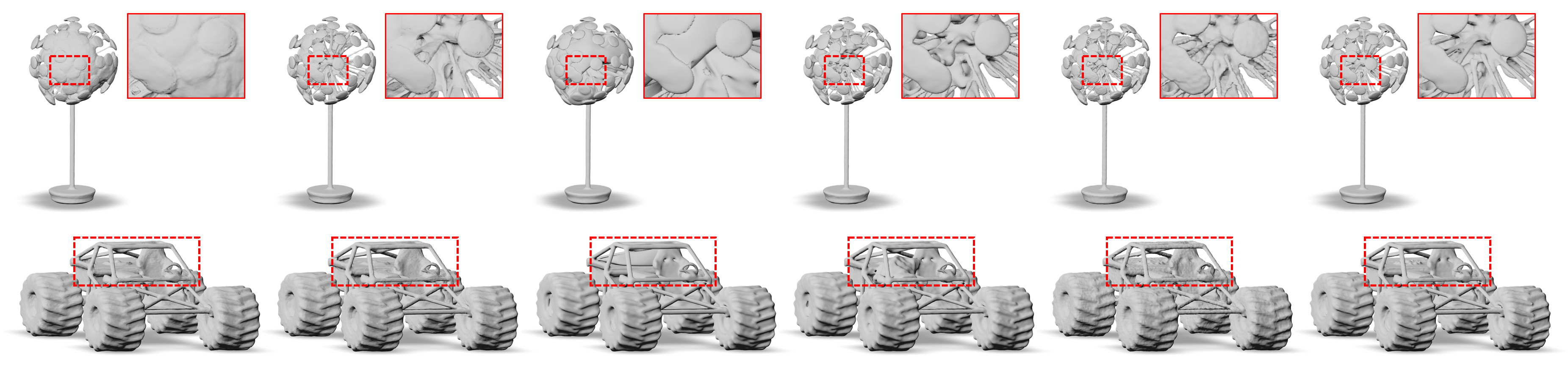}
               \put(3.1, -1.3){\small{CAP-UDF} \cite{zhou2024cap}}
                \put(19.0,-1.3){\small{LevelSetUDF} \cite{zhou2023learning}}
                \put(39.0, -1.3){\small{NSH} \cite{wang2023neural}}
                \put(54.9, -1.3){\small{StEik} \cite{yang2024stabilizing}}
                \put(70.5, -1.3){\small{\revise{PG-SDF}} \cite{koneputugodage2024small}}
                \put(88.5, -1.3){\small{Ours}}
    \end{overpic}
    \caption{Visual comparison of surface reconstruction under ShapeNet dataset \cite{chang2015shapenet} processed by DISN\cite{xu2019disn}.}
    \label{fig:shapnet process by DISN}
\end{figure*}
\begin{figure*}[htbp]
    \centering
    \begin{overpic}[width=\textwidth]{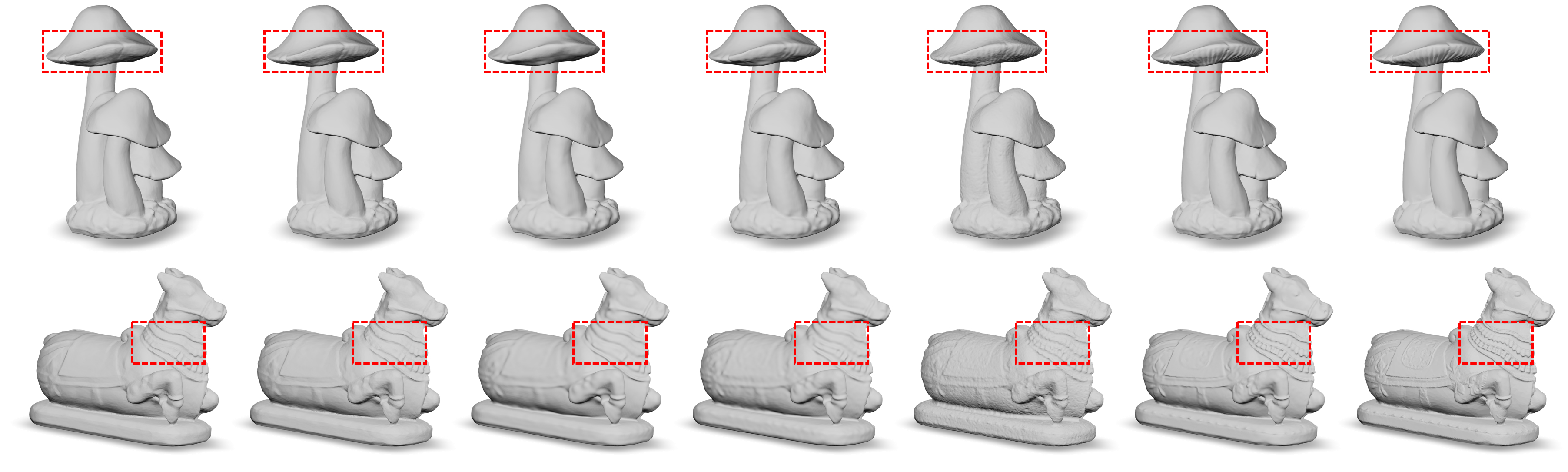}
               \put(2.8, -1){\small{CAP-UDF} \cite{zhou2024cap}}
                \put(16.5,-1){\small{LevelSetUDF} \cite{zhou2023learning}}
                \put(33.5, -1){\small{NSH} \cite{wang2023neural}}
                \put(46.7, -1){\small{StEik} \cite{yang2024stabilizing}}
                \put(60, -1){\small{\revise{PG-SDF}} \cite{koneputugodage2024small}}
                \put(76.0, -1){\small{Ours}}
                \put(90, -1){\small{GT}}
    \end{overpic}
    \caption{Visual comparison of surface reconstruction under Thingi10K dataset \cite{zhou2016thingi10k}.}
    \label{fig:Thingi10K}
\end{figure*}
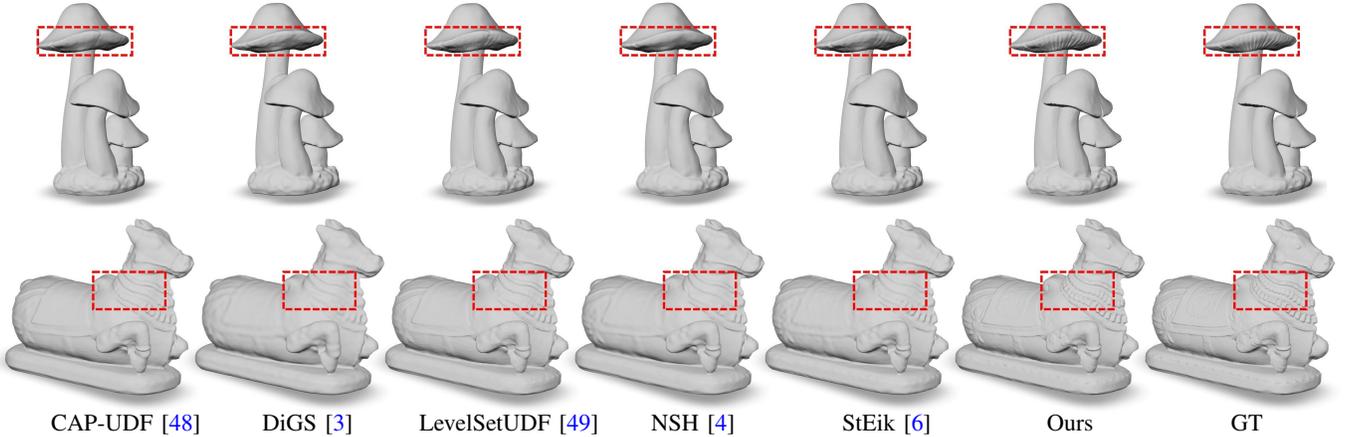

\vspace{1em}
\noindent\textbf{Thingi10K and Stanford 3D Scanning Repository.}
The Thingi10K dataset \cite{zhou2016thingi10k} comprises a variety of models with intricate geometric details. We use its Thingi32 subset. Additionally, we select six shapes from the Stanford 3D Scanning Repository: Armadillo, Bunny, Dragon, Asian Dragon, Happy Buddha, and Thai Statue. Among these, the Bunny is not watertight. For each shape, we randomly sample 50k points as the input point cloud. The quantitative comparisons are reported in Tab. \ref{table:Thingi10K}. Furthermore, the visual comparisons in Fig. \ref{fig:Thingi10K} and Fig. \ref{fig:stanford} illustrate that our method captures more accurate and finer geometric details such as threads of objects. Especially, for the Bunny, all compared methods erroneously fill in the holes at the bottom. Only our method accurately reconstructs the character.

\begin{table}[ht]
  \caption{Quantitative comparison of surface reconstruction under ShapeNet dataset \cite{chang2015shapenet} processed by DISN \cite{xu2019disn}.}
  \label{table:ShapeNet processed by DISN}
  \centering{
  \resizebox{\linewidth}{!}{
  \begin{tabular}{l|cccccc}
    \toprule
    \multirow{2}{*}{Method} & \multicolumn{2}{c}{CD $\downarrow$} &\multicolumn{2}{c}{Normal $\uparrow$}  &\multicolumn{2}{c}{F-Score $\uparrow$}\\
    &mean&median&mean&median&mean&median\\
    \midrule
    CAP-UDF \cite{zhou2024cap}&4.218&3.690&96.77&97.85&93.68&97.34\\
    LevelsetUDF \cite{zhou2023learning}&3.650&3.569&97.34&98.10&94.71&98.21\\
     \midrule
    IGR \cite{gropp2020implicit}&16.19&6.699&91.45&93.60&77.47&83.85\\
    SIREN \cite{sitzmann2020implicit}&6.848&5.812&95.15&95.64&84.62&88.01\\
    DiGS \cite{ben2022digs}&4.260&3.973&96.94&97.94&93.52&96.55\\
    NSH \cite{wang2023neural}&4.616&4.008&96.98&98.18&92.64&96.05\\
    StEik \cite{yang2024stabilizing}&4.391&3.659&97.25&98.20&93.87&97.70\\
    \revise{PG-SDF \cite{koneputugodage2024small}} &\revise{4.147}&\revise{3.872}&\revise{96.75}&\revise{97.74}&\revise{93.31}&\revise{96.66} \\
     \midrule
    Ours &\textbf{3.525}&\textbf{3.497}&\textbf{97.79}&\textbf{98.42}&\textbf{95.34}&\textbf{98.45}\\
    \bottomrule
\end{tabular}}}
\end{table}
\begin{table*}[htbp]
  \caption{Quantitative comparisons of surface reconstruction on Thingi10K dataset \cite{zhou2016thingi10k} and the Stanford 3D Scanning Repository.}
  \label{table:Thingi10K}
  \centering{
  \resizebox{\linewidth}{!}{
  \begin{tabular}{l|cccccc|cccccc}
    \toprule
    \multirow{3}{*}{Method} &\multicolumn{6}{c|}{Thingi10K}&\multicolumn{6}{c}{Stanford 3D Scanning Repository}\\
    \cmidrule{2-13}
    & \multicolumn{2}{c}{CD $\downarrow$} &\multicolumn{2}{c}{Normal $\uparrow$}  &\multicolumn{2}{c|}{F-Score $\uparrow$} &  \multicolumn{2}{c}{CD $\downarrow$} &\multicolumn{2}{c}{Normal $\uparrow$}  &\multicolumn{2}{c}{F-Score $\uparrow$}\\    &mean&median&mean&median&mean&median&mean&median&mean&median&mean&median\\
    \midrule
    CAP-UDF \cite{zhou2024cap}&4.074&4.024&98.15&98.44&94.67&95.90&3.807&3.847&94.81&96.84&97.10&97.57\\
    LevelsetUDF \cite{zhou2023learning}&3.889&3.892&98.40&98.50&95.04&96.52&3.668&3.652&95.22&96.96&97.42&98.03\\
     \midrule
    IGR \cite{gropp2020implicit}&10.711&4.564&95.90&97.70&87.25&92.75&15.75&4.313&90.49&93.37&85.77&94.33\\
    SIREN \cite{sitzmann2020implicit}&10.198&11.398&94.91&94.44&77.91&73.94&8.180&7.801&92.27&93.03&82.08&83.37\\
    DiGS \cite{ben2022digs}&4.309&4.095&98.28&98.41&93.96&95.96&3.841&3.729&95.51&97.00&97.29&98.27\\
    NSH \cite{wang2023neural}&4.338&4.087&98.70&98.89&94.67&95.93&3.677&3.588&96.36&\textbf{97.93}&97.70&98.56\\
    StEik \cite{yang2024stabilizing}&4.048&4.032&98.67&98.86&94.57&96.28&3.675&3.583&96.30&97.72&97.62&98.51\\
    \revise{PG-SDF \cite{koneputugodage2024small}} &\revise{4.102}&\revise{3.980}&\revise{97.94}&\revise{98.26}&\revise{93.98}&\revise{96.15}&\revise{4.183}&\revise{3.767}&\revise{90.38}&\revise{94.67}&\revise{92.29}&\revise{97.03} \\
     \midrule
    Ours &\textbf{3.848}&\textbf{3.862}&\textbf{98.86}&\textbf{98.95}&\textbf{95.16}&\textbf{96.60}&\textbf{3.499}&\textbf{3.438}&\textbf{96.56}&97.91&\textbf{97.91}&\textbf{98.59}\\
    \bottomrule
\end{tabular}}}
\end{table*}

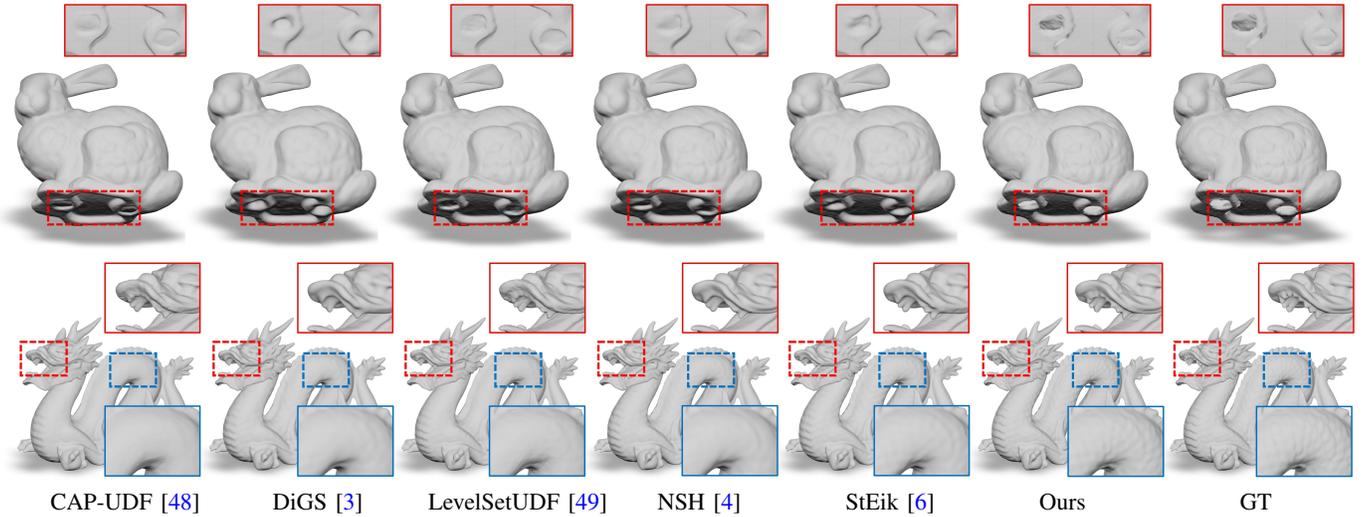
\begin{figure*}[htbp]
    \centering
    \begin{overpic}[width=\textwidth]{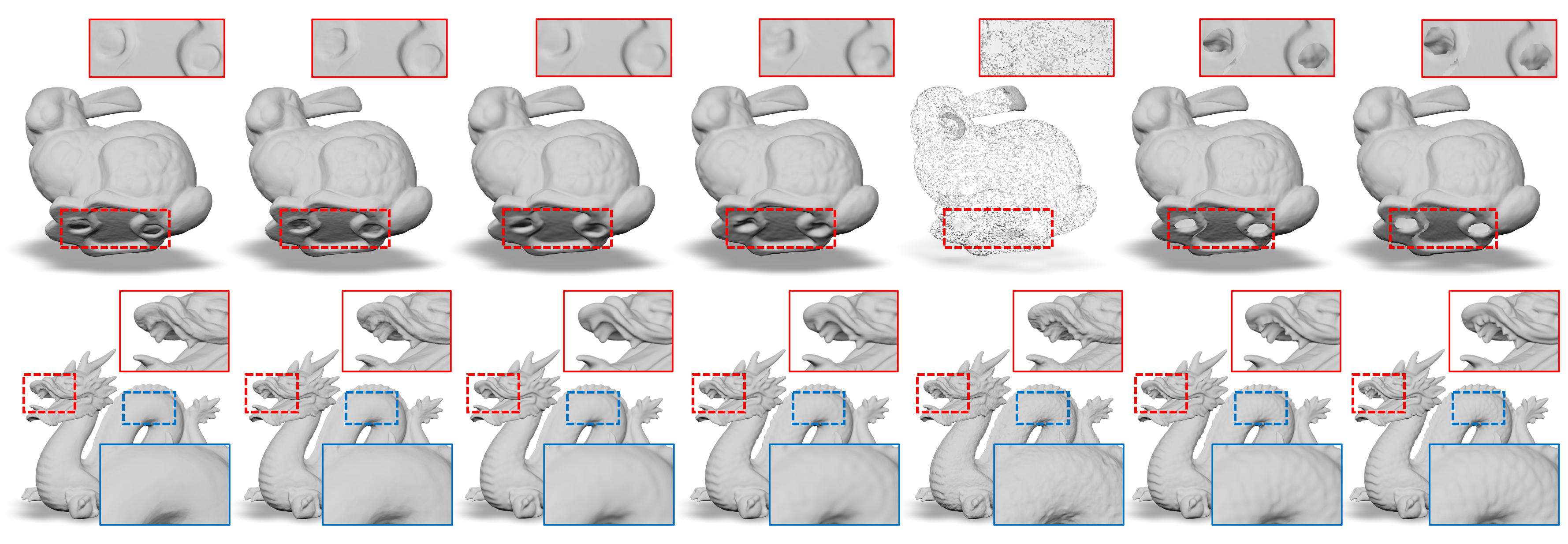}%
               \put(2.8, -1.0){\small{CAP-UDF} \cite{zhou2024cap}}
                \put(16.0,-1.0){\small{LevelSetUDF} \cite{zhou2023learning}}
                \put(33.5, -1.0){\small{NSH} \cite{wang2023neural}}
                \put(47.2, -1.0){\small{StEik} \cite{yang2024stabilizing}}
                \put(60.3, -1.0){\small{\revise{PG-SDF}} \cite{koneputugodage2024small}}
                \put(76.2, -1.0){\small{Ours}}
                \put(90.5, -1.0){\small{GT}}
    \end{overpic}
    \caption{Visual comparison of surface reconstruction under the Stanford 3D Scanning Repository.}
    \label{fig:stanford}
\end{figure*}

\subsection{Ablation Studies}
\label{subsec:ablation}
\textbf{Loss Functions.}
We use different combinations of loss functions to validate the necessity of each loss term. We report the quantitative comparison in Tab. \ref{table:loss} and provide the visual comparison in Fig. \ref{fig:ablation}. The result demonstrates that the Monge-Ampere loss $L_{MA}$ is the core loss of our method and its absence would result in reconstruction failure. $L_{Dirichlet}$ and $L_{Neumann}$ each contribute to enhancing the reconstruction accuracy. Although the impact of $L_{non-manifold}$ on performance enhancement is limited, it aids in stabilizing network training and ensures a uniform set of loss weights across the majority of objects in our experiments.
\begin{table}[htbp]
  \caption{Quantitative comparison about different combinations of loss functions. "-" means reconstruction failure.}
  \label{table:loss}
  \centering{
  \resizebox{\linewidth}{!}{
  \begin{tabular}{l|cccccc}
    \toprule
    \multirow{2}{*}{Loss Function} & \multicolumn{2}{c}{CD $\downarrow$} &\multicolumn{2}{c}{Normal $\uparrow$}  &\multicolumn{2}{c}{F-Score $\uparrow$}\\
    &mean&median&mean&median&mean&median\\
    \midrule
    $L_{Dir}$&--&--&--&--&--&--\\
    $L_{Neu}$&--&--&--&--&--&--\\
    $L_{Dir}+L_{Neu}$&--&--&--&--&--&--\\
    $L_{Dir}+L_{MA}$&4.533&4.584&91.74&93.35&91.43&91.47\\
    $L_{Neu}+L_{MA}$&3.738&3.594&\textbf{96.61}&97.89&97.64&98.43\\    $L_{Dir}+L_{Neu}+L_{MA}$&\textbf{3.499}&\textbf{3.438}&96.50&97.89&\textbf{97.93}&\textbf{98.60}\\    $L_{Dir}+L_{Neu}+L_{MA}+L_{non}$&\textbf{3.499}&\textbf{3.438}&96.56&\textbf{97.91}&97.91&98.59\\
    \bottomrule
\end{tabular}}}
\end{table}
\begin{figure}[tbp]
    \centering
     \begin{overpic}[width=0.5\textwidth]{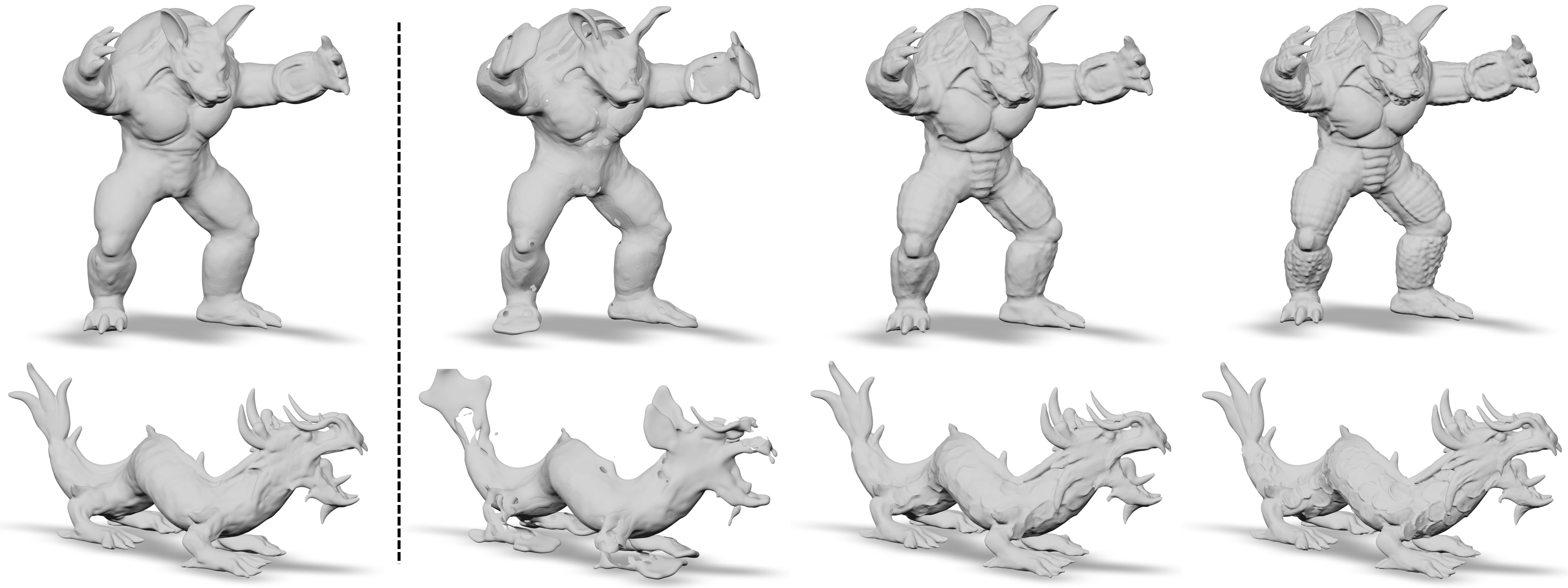}
               \put(2.0, -3){\small{w/o $L_{Neumann}$}}
               \put(32.5,-3){\small{$K=100$}}
               \put(57,-3){\small{$K=2000$}}
               \put(82,-3){\small{$K=1000$}}
    \end{overpic}
    \caption{Visualization of reconstruction results with different values of $K$ and combinations of loss functions.}
    \label{fig:ablation}
\end{figure}
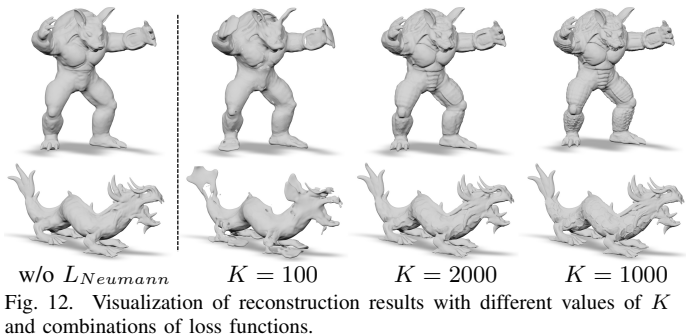
\vspace{1em}

\noindent\textbf{The Irreplaceability of Monge-Ampere Regularization.}
Based on Eq. \eqref{equation:first-order2}, we define the following loss function:
\begin{equation}
L_{Eikonal}^{'} = \int_{\mathcal{P} \cup \mathcal{Q}} \left| \Vert \nabla f(\boldsymbol{x};\theta) \Vert_2^2 - 4K \cdot f(\boldsymbol{x};\theta) \right| d\boldsymbol{x},
\end{equation}
where $\mathcal{P}$ represents the input point cloud and $\mathcal{Q}$ is a set of sample points from space. This term is similar to the Eikonal term used in SDF-based approaches. However, $t(\boldsymbol{x}) = 0$ is a trivial solution to Eq. \eqref{equation:first-order2}, meaning that when $f(\boldsymbol{x};\theta)$ outputs a constant value of 0, the loss also becomes 0. This issue is alleviated by adding the term $L_{non-manifold}$. Furthermore, since this loss encourages $\Vert \nabla f(\boldsymbol{x};\theta) \Vert_2^2 = 4K \cdot f(\boldsymbol{x};\theta)$, rather than directly encouraging $\Vert \nabla f(\boldsymbol{x};\theta) \Vert_2^2$ to match the ground-truth values, and stochastic gradient descent typically leads to a local optimum, the converged network will exhibit the scenario where $\Vert \nabla f(\boldsymbol{x};\theta) \Vert_2^2 \approx 4K \cdot f(\boldsymbol{x};\theta)$, but neither will match the true values. Consequently, as shown in Fig. \ref{fig:Eikonal}, when the term $L_{Eikonal}^{'}$ is used in place of Monge-Ampere regularization in the total loss function, the network cannot effectively learn the S\textsuperscript{2}DF. On the contrary, our Monge-Ampere regularization fundamentally resolves the issue of $f(\boldsymbol{x};\theta)$ outputting a constant value of 0 and directly encourages $Det(H_f(\boldsymbol{x}) - 2K \cdot I)$ to match the ground-truth values.
\begin{figure}[tbp]
    \centering
     \begin{overpic}[width=0.5\textwidth]{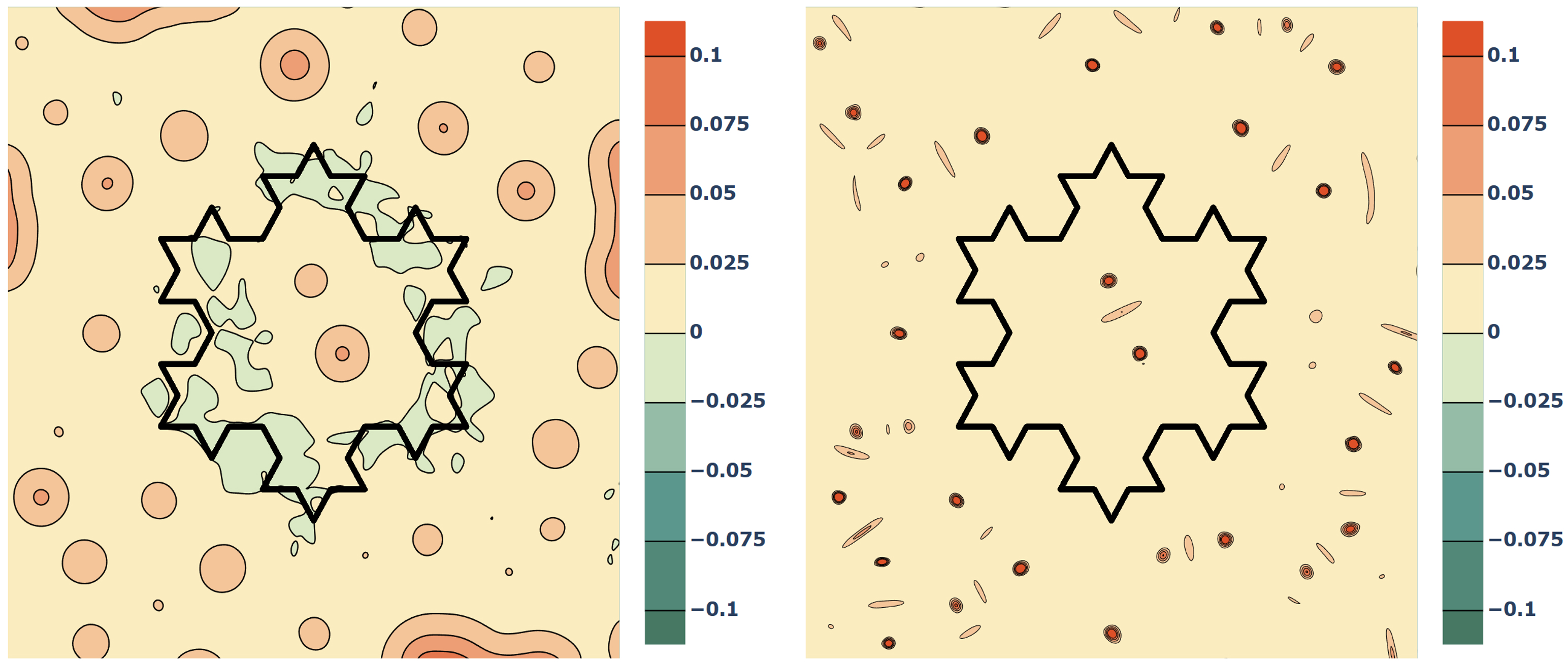}
        \put(14.2, -3.8){\small{$f(\boldsymbol{x};\theta)$}}
         \put(64.5,-3.8){\small{$L_{Eikonal}^{'}$}}
    \end{overpic}
    \vspace{-3mm}
    \caption{Visualization of reconstruction results when $L_{Eikonal}^{'}$ is used instead of Monge-Ampere regularization in the total loss function. The bold line represents the target curve.}
    \label{fig:Eikonal}
\end{figure}
\begin{table}[t]
  \caption{Quantitative comparison about different values of $K$. "-" means reconstruction failure.}
  \label{table:K}
  \centering{
  \resizebox{\linewidth}{!}{
  \begin{tabular}{l|cccccc}
    \toprule
    \multirow{2}{*}{$K$} & \multicolumn{2}{c}{CD $\downarrow$} &\multicolumn{2}{c}{Normal $\uparrow$}  &\multicolumn{2}{c}{F-Score $\uparrow$}\\
    &mean&median&mean&median&mean&median\\
    \midrule
    1&-&-&-&-&-&-\\
    100&6.209&6.585&86.83&88.67&79.99&78.85\\
    500&3.557&3.488&95.98&97.50&97.86&98.54\\    1000&\textbf{3.499}&\textbf{3.438}&\textbf{96.56}&\textbf{97.91}&\textbf{97.91}&\textbf{98.59}\\
    2000&3.634&3.564&95.60&97.22&97.73&98.51\\
    \bottomrule
\end{tabular}}}
\end{table}

\begin{table}[t]
  \caption{\revise{Quantitative comparison under different network initialization schemes.}}
  \label{table:initialization}
  \centering{
  \resizebox{\linewidth}{!}{
  \begin{tabular}{l|cccccc}
    \toprule
    \multirow{2}{*}{Initialization Scheme} & \multicolumn{2}{c}{CD $\downarrow$} &\multicolumn{2}{c}{Normal $\uparrow$}  &\multicolumn{2}{c}{F-Score $\uparrow$}\\
    &mean&median&mean&median&mean&median\\
    \midrule
    MFGI \cite{ben2022digs}&3.652&3.600&95.14&96.64&97.66&98.32\\
    SIREN \cite{sitzmann2020implicit}&\textbf{3.499}&\textbf{3.438}&\textbf{96.56}&\textbf{97.91}&\textbf{97.91}&\textbf{98.59}\\
    \bottomrule
\end{tabular}}}
\end{table}

\begin{figure}[tp]
    \centering
     \begin{overpic}[width=0.5\textwidth]{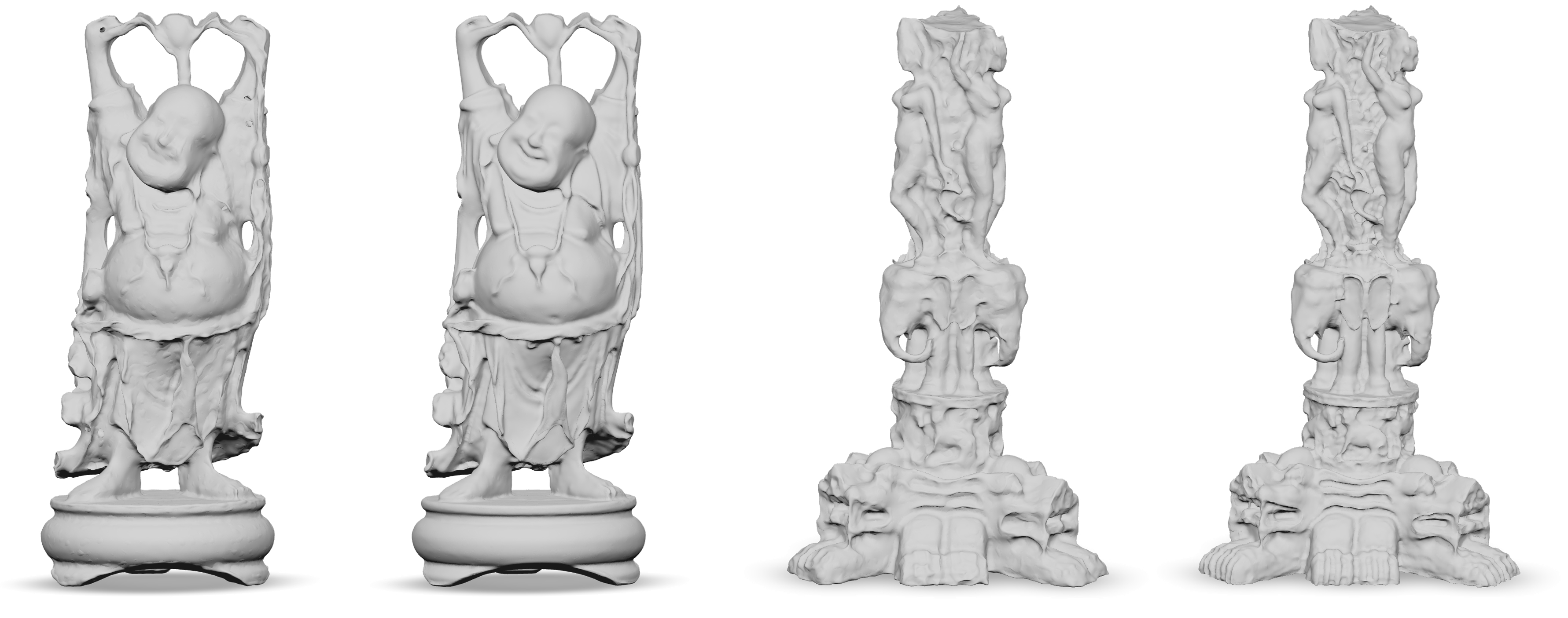}
               \put(4.4, -3){\small{MFGI \cite{ben2022digs}}}
               \put(28.0,-3){\small{SIREN \cite{sitzmann2020implicit}}}
               \put(55.2,-3){\small{MFGI \cite{ben2022digs}}}
               \put(80.5,-3){\small{SIREN \cite{sitzmann2020implicit}}}
    \end{overpic}
    \caption{\revise{Visualization of reconstruction results under different network initialization schemes.}}
    \label{fig:initialization}
\end{figure}
\vspace{1em}
\noindent\textbf{The Effect of $K$.}
To validate the effectiveness of $K$ and determine its optimal value, we set $K$ to different values to conduct experiments. Based on Tab. \ref{table:K} and Fig. \ref{fig:ablation}, the reconstruction results with $K$ set to 1000 achieve the best performance among all the tested values. When $K$ is set to 1, as the ground-truth S\textsuperscript{2}DF values for the points near the zero-level set are very close, the neural network fails to effectively distinguish these points, making it difficult to extract the zero-level set from it. In contrast, setting 
$K$ to 1000 significantly amplifies the differences between the S\textsuperscript{2}DF values of points near the zero-level set, facilitating more accurate learning of the S\textsuperscript{2}DF.

The experiments offer a general guideline for users to set $K$. If $K$ is assigned a small value, it cannot adequately address the issue of S\textsuperscript{2}DF values being extremely close to each other near the zero-level set. The value of $K$ cannot be arbitrarily large either. The neural network is trained to have its Hessian possess an eigenvalue of $2K$. If the value of $K$ is too large, the ideal convergent network’s Hessian will have a very large eigenvalue. However, the eigenvalues of the network's Hessian are relatively small at initialization. Hence, the larger the value of $K$, the greater the difference between the initialized network and the ideal convergent network, making the training more difficult. This explains the decrease in reconstruction quality when $K$ is set to 2000.

\vspace{1em}
\noindent\revise{\textbf{Network Initialization Scheme.} To identify the optimal initialization scheme and evaluate the convergence of our method under different initialization settings, we conduct experiments using two SIREN \cite{sitzmann2020implicit} initialization schemes: (1) the original scheme proposed by SIREN and (2) the multi-frequency geometric initialization (MFGI) introduced by DiGS \cite{ben2022digs}. The quantitative results in Tab. \ref{table:initialization} and the visual comparisons in Fig. \ref{fig:initialization} demonstrate that our method successfully converges to S\textsuperscript{2}DF under both initialization schemes. The SIREN-proposed initialization outperforms MFGI, as MFGI initializes the network to approximate the SDF of a sphere, which differs significantly from the S\textsuperscript{2}DF we aim to learn. In contrast, the SIREN-proposed initialization strategy offers greater versatility, adapting to a broader  range of tasks.}

\subsection{Runtime Performance}
Finally, we compare the time consumption of our method for learning one shape with SIREN \cite{sitzmann2020implicit}, CAP-UDF \cite{zhou2024cap}, NSH \cite{wang2023neural}, LevelSetUDF \cite{zhou2023learning} and StEik \cite{yang2024stabilizing} using the object-level Stanford 3D Scanning Repository and the scene-level 3D Scene dataset \cite{zhou2013dense}. Tab. \ref{table:time} records the time costs for each method. The result demonstrates that our time cost is higher than that of SDF-based methods utilizing the Eikonal term, such as SIREN and NSH. Because computing the determinant of the Hessian is complex and increases the overhead of back-propagation. However, compared to UDF-based methods such as CAP-UDF and LevelSetUDF, our method uses less time to achieve better performance. Particularly, for scene-level models, we maintain the same experimental setup as for object-level models, resulting in nearly identical time requirements for both. However, CAP-UDF and LevelSetUDF require more iterations for scene-level models compared to object-level models, leading to longer training times.
\begin{table}[htbp]
  \caption{Runtime comparison in minutes for learning one shape.}
  \label{table:time}
  \resizebox{\linewidth}{!}{
  \begin{tabular}{c|cc|ccc|c}
    \toprule
     &CAP-UDF \cite{zhou2024cap}&LevelSetUDF \cite{zhou2023learning}&SIREN \cite{sitzmann2020implicit}&NSH \cite{wang2023neural}&StEik\cite{yang2024stabilizing}&Ours\\
    \midrule
    Object-level &15.94&18.65&3.205&10.38&6.008&12.54\\
    Scene-level &23.04&50.44&-&-&-&12.56\\
    \bottomrule
\end{tabular}}
\end{table}

%% file: Section/limitation_conclusion.tex
\section{Limitations}
\label{limitation}
First, as discussed in previous sections, our method effectively reconstructs surfaces of various types and accurately captures the geometric details present in the input point clouds. However, its high fidelity to the input data and ability to represent arbitrary shapes may lead to misinterpretations of gaps between neighboring points as intrinsic properties of the underlying surfaces when handling sparse point clouds. This results in holes in the reconstructed surfaces. Upsampling the input point clouds prior to reconstruction can effectively mitigate this limitation. In the future, we will explore more direct and effective approaches to enhance our method’s performance with sparse point clouds. 

\revise{Second, our Dirichlet, Neumann, and Monge-Ampere losses directly constrain the function values, gradients, and Hessians of the network at input points, respectively, allowing our method to learn an accurate zero-level set. However, the network’s outputs at other spatial locations are regulated solely by the Monge-Ampere loss, and our method is not tailored to explicitly preserve shape features. Consequently, the learned implicit field exhibits a certain degree of over-smoothing in structures like the corners of a snowflake.}

\revise{Third, since we randomly sample points from the input point cloud during training, the sampled boundary points for non-closed shapes, such as circular arcs, vary slightly in each iteration. This results in suboptimal learning performance at the boundaries of such shapes. A sampling method that maintains stable boundary points across multiple samples is worth exploring in the future.}

\revise{Last but not least, our method is sensitive to the value of $K$, which typically needs to be relatively large. Although the initialization strategy proposed by SIREN \cite{sitzmann2020implicit} is feasible, it is not specifically designed for learning S\textsuperscript{2}DF. In particular, it initializes the eigenvalues of the network's Hessian to very small values, making it less effective for learning S\textsuperscript{2}DF with large $K$. Developing a more effective approach to address the slow function value variation of the squared distance function near the target surface and designing a network initialization strategy specifically for learning S\textsuperscript{2}DF are promising directions for future research.}

\section{Conclusion}
\label{conclusion}
In this paper, we have introduced a novel implicit representation called  S\textsuperscript{2}DF, which is capable of modeling arbitrary shapes while maintaining differentiability at the zero-level set. We analyzed the mathematical properties of S\textsuperscript{2}DF and subsequently proposed a method to effectively learn S\textsuperscript{2}DF from raw unoriented point clouds by Monge-Ampere regularization, which encourages the Hessian of the neural network to consistently possess an eigenvalue of $2K$. \revise{Thanks to the differentiability of the S\textsuperscript{2}DF at the zero-level set, our approach, compared to UDF-based methods, not only constrains the function values at the input point cloud during neural network learning, but also imposes constraints on the first- and second-order derivatives, which enables the learning of a more accurate zero-level set.} Extensive experiments on multiple datasets demonstrate that our method, which does not require ground-truth distance values as supervision during training, can reconstruct arbitrary surfaces and achieve higher accuracy than the existing state-of-the-art approaches.